\documentclass[sigconf]{acmart}
\AtBeginDocument{%
  }

\usepackage{amsmath,amsthm}
\usepackage{algorithm} 
\usepackage{algpseudocode}
\usepackage{graphicx}
\usepackage{subfigure}
\usepackage{multirow}
\usepackage{stfloats}

\usepackage{booktabs}
\usepackage{threeparttable}
\setcopyright{acmlicensed}
\copyrightyear{2018}
\acmYear{2018}
\acmDOI{XXXXXXX.XXXXXXX}
\acmConference[Conference acronym 'XX]{Make sure to enter the correct
  conference title from your rights confirmation email}{June 03--05,
  2018}{Woodstock, NY}

\begin{document}

\title{Geometry-aware Incremental Neural Operator for Long-Horizon PDE prediction}

\settopmatter{authorsperrow=5}

\author{Jiaquan Zhang}
\affiliation{%
  \institution{School of Information and Software Engineering, UESTC}
  \city{Chengdu}
  \country{China}
}

\author{Shuxu Chen}
\affiliation{%
  \institution{Electronics and Information Convergence Engineering, KHU}
  \city{Yongin-si}
  \country{Korea}}

\author{Haifan Meng}
\affiliation{%
  \institution{School of Computer Science and Engineering, UESTC}
  \city{Chengdu}
  \country{China}
}

\author{Yi Lu}
\affiliation{%
 \institution{Department of Mathematical Sciences, UOL}
 \city{Liverpool}
 \country{England}}

\author{Zhihan Lyu}
\affiliation{%
  \institution{School of Computer Science and Technology, XDU}
  \city{Shanxi}
  \country{China}}

\author{Fan Mo}
\affiliation{%
  \institution{School of Mechanical and Electrical Engineering, UESTC}
  \city{Chengdu}
  \country{China}}

\author{Wei Dong}
\affiliation{%
  \institution{College of Computer and Information Engineering, XAUAT}
  \city{Xi'an}
  \country{China}}

\author{Yang Yang}
\affiliation{%
  \institution{School of Computer Science and Engineering, UESTC}
  \city{Chengdu}
  \country{China}
}

\author{Chaoning~Zhang}
\affiliation{%
  \institution{School of Computer Science and Engineering, UESTC}
  \city{Chengdu}
  \country{China}
}

\renewcommand{\shortauthors}{Zhang et al.}

\begin{abstract}
Neural operators have shown strong potential for learning solution operators of partial differential equations (PDEs).
However, long-horizon autoregressive prediction remains challenging: local errors accumulate as spectral inconsistency, phase misalignment, or mean drift. 
Existing methods mainly improve state representations and operator backbones, while leaving the repeatedly applied
latent transition increment weakly structured, allowing spectral errors
and unstable channel couplings to accumulate during rollout.
To address these issues, we propose a geometry-aware incremental neural operator (GeoIncNO) for stable long-horizon PDE prediction. 
GeoIncNO predicts latent increments for residual advancement and
uses lightweight low-rank projectors to regulate channel coupling
within active frequency bands derived from the increment spectral
energy distribution.
To reduce physical-space reconstruction errors, GeoIncNO further introduces a mean--fluctuation decoupled reconstruction mechanism, where stable mean structures and dynamic fluctuations are fused separately, and phase correction is applied only to the zero-mean fluctuation component. 
Extensive experiments on six PDE benchmarks, covering 1D, 2D, and 3D dynamical systems, show that GeoIncNO achieves consistently strong prediction accuracy, improved rollout stability, and better spectral fidelity compared with competitive neural-operator baselines.
\end{abstract}


\received{20 February 2007}
\received[revised]{12 March 2009}
\received[accepted]{5 June 2009}

\maketitle

\section{Introduction}
Partial differential equations (PDEs) describe the spatiotemporal evolution of many complex physical systems, such as fluid flows, heat transfer and multiscale materials \cite{yip2022step,brunton2024promising,aarts2001neural}. 
Although traditional numerical solvers are accurate, they often require fine-grid discretization and iterative computation, which becomes costly in high-dimensional and long-horizon settings \cite{karniadakis2021physics}.
Neural operators offer a data-driven paradigm for PDE modeling by directly learning mappings between function spaces \cite{azizzadenesheli2024neural,liu2025architectures,dai2026learning,ren2026foundation}. 
This enables fast prediction of future physical fields from initial conditions, making neural operators an important tool for scientific machine learning \cite{DBLP:conf/iclr/LiKALBSA21}.
Despite the significant progress of existing neural operators in one-step and short-term prediction, long-horizon autoregressive prediction remains challenging due to error accumulation \cite{brandstetter2022message,mccabe2023towards,worrall2024spectral}.
Since the prediction at each step is repeatedly fed back as the input for subsequent steps, local errors can be amplified during rollout, leading to degraded spectral consistency and unstable temporal evolution \cite{wu2025differential}.

Existing neural PDE predictors can be broadly grouped according to where temporal evolution is modeled.
One class learns the solution evolution in the physical grid space or in lifted operator features, without explicitly formulating a compact latent-state transition process.
These methods typically improve prediction by introducing Fourier parameterization, multiscale structures, local convolutions, or stronger spatiotemporal mixing modules~\cite{DBLP:journals/tmlr/RahmanRA23,hu2024better}.
Another class first maps physical fields into a latent space and then learns the evolution of latent states, thereby reducing the modeling complexity in the original physical space~\cite{wang2024lno,tiwari2025latent,chen2025neural}.
Although latent-space neural operators provide a compact representation space for PDE dynamics, existing formulations usually model the evolution by mapping one latent state to the next.
The latent change between consecutive states remains entangled with the state representation, so its frequency-domain structure and channel-wise coupling are difficult to characterize or regulate. This entanglement is particularly problematic under autoregressive rollout, where the same transition update is applied repeatedly and small per-step errors accumulate over long horizons.
Therefore, stable long-horizon prediction requires not only an effective latent-state representation, but also a structured increment formulation that directly regularizes the repeated transition process.
Motivated by this limitation, we examine latent-state propagation
from the perspective of the transition increment.
As shown in Section~\ref{sec:motivation}, latent states and latent
increments exhibit different dominant spatial responses.
Specifically, the latent state \(z_t\) retains broad spatial
organization, whereas the transition increment \(\delta z_t\)
exhibits a distinct response associated with the change between
consecutive latent states
(Figure~\ref{fig:latent_pca}).
The covariance spectrum of \(\delta z_t\) also decays faster than
that of \(z_t\), with more of its variation captured by the leading
channel directions
(Figure~\ref{fig:covariance_geometry}).
In addition, its spectral energy is distributed non-uniformly across
spatial frequencies
(Figure~\ref{fig:active_bands}).
Together, these observations motivate modeling the latent increment
explicitly in the frequency--channel domain.
We refer to this modeling perspective as increment geometry.
GeoIncNO accordingly applies structured modulation to the increment
before residual state advancement.

Building on this increment-centered perspective, we propose a
Geometry-aware Incremental Neural Operator (GeoIncNO) for
long-horizon PDE prediction.
GeoIncNO first encodes the input physical field into a latent state
and uses a latent-space backbone to predict a raw transition
increment.
Instead of directly applying this increment for state update, GeoIncNO introduces Active-Band Projection (ABP), which constructs active frequency bands from the spectral structure of the increment and uses lightweight low-rank geometric projectors to modulate band-wise channel coupling.
The resulting geometrically shaped increment is then used for residual latent-state advancement, allowing the frequency-band structure to directly participate in forward propagation and mitigate unstable spectral components and redundant channel coupling that may otherwise accumulate during rollout.
GeoIncNO further introduces Mean--Fluctuation Decoupled Reconstruction (MFDR) to address physical-space reconstruction errors.
Since the candidate physical predictions may contain slow mean shifts, dynamic fluctuations, and phase misalignment, directly fusing complete predicted fields can entangle these errors.
MFDR therefore uses two complementary prediction branches, decomposes their outputs into temporal mean and zero-mean fluctuation fields, and fuses the two components separately.
Phase correction is applied only to the zero-mean fluctuation field, reducing interference with the mean structure and alleviating mean drift during long-horizon prediction.

The main contributions of this work are summarized as follows.
\textbf{\textit{(i)}} We introduce an increment-centered perspective for
neural PDE operators, where the latent transition increment is explicitly
modeled as the object governing autoregressive state propagation.
\textbf{\textit{(ii)}} We propose an active-band increment projection
mechanism that constructs frequency bands from the spectral energy of latent
increments and applies lightweight low-rank projectors to shape band-wise
channel geometry before latent-state advancement.
\textbf{\textit{(iii)}} We design a mean--fluctuation decoupled
reconstruction strategy that separately fuses stable mean structures and
zero-mean fluctuations, while restricting phase correction to the
fluctuation field to reduce mean-field drift.
\textbf{\textit{(iv)}} We validate GeoIncNO on six PDE benchmarks spanning
1D, 2D, and 3D systems, demonstrating improved prediction accuracy, spectral
fidelity, and rollout stability over competitive neural-operator
baselines. Ablation studies further confirm the contribution of each
component.

\section{Related Work}
\label{app:related}
Neural operators~\cite{kovachki2023neural} learn mappings between function spaces, unlike physics-informed neural networks~\cite{raissi2019pinns} that embed governing equations into the loss and retrain per instance; FNO~\cite{DBLP:conf/iclr/LiKALBSA21} and DeepONet~\cite{lu2021deeponet} are representative. Subsequent work either strengthens the operator backbone, via adaptive Fourier mixing~\cite{guibas2022afno}, spectral-basis enrichment~\cite{gupta2021multiwavelet,hu2025wavelet}, attention-based mixing~\cite{hao2023gnot}, and Koopman-linearized operators~\cite{xiong2024koopman}, or compresses PDE dynamics into a latent space~\cite{lusch2018koopman}, as in LSM~\cite{wu2023lsm}, LNO~\cite{wang2024lno}, and PI-Latent-NO~\cite{karumuri2025pilatentno}. Both routes model the state-to-state or latent-state mapping, however, leaving the latent increment that drives temporal evolution entangled with the state representation rather than modeled as a structured object in its own right.
We provide the complete related work in Appendix~\ref{app:RL}.

\section{Motivation}
\label{sec:motivation}

This section presents the core empirical observations that motivate
GeoIncNO's latent-increment modeling and active-band construction.
We first show that latent increments differ from full latent states
in spatial response and covariance geometry.
We then analyze their non-uniform spectral energy distribution,
which motivates energy-adaptive frequency partitioning.
Additional analyses of frequency-dependent band geometry and
mean--fluctuation separation are provided in
Appendix~\ref{sec:moti}.


\begin{figure}[t]
    \centering
    \includegraphics[width=0.9\linewidth]{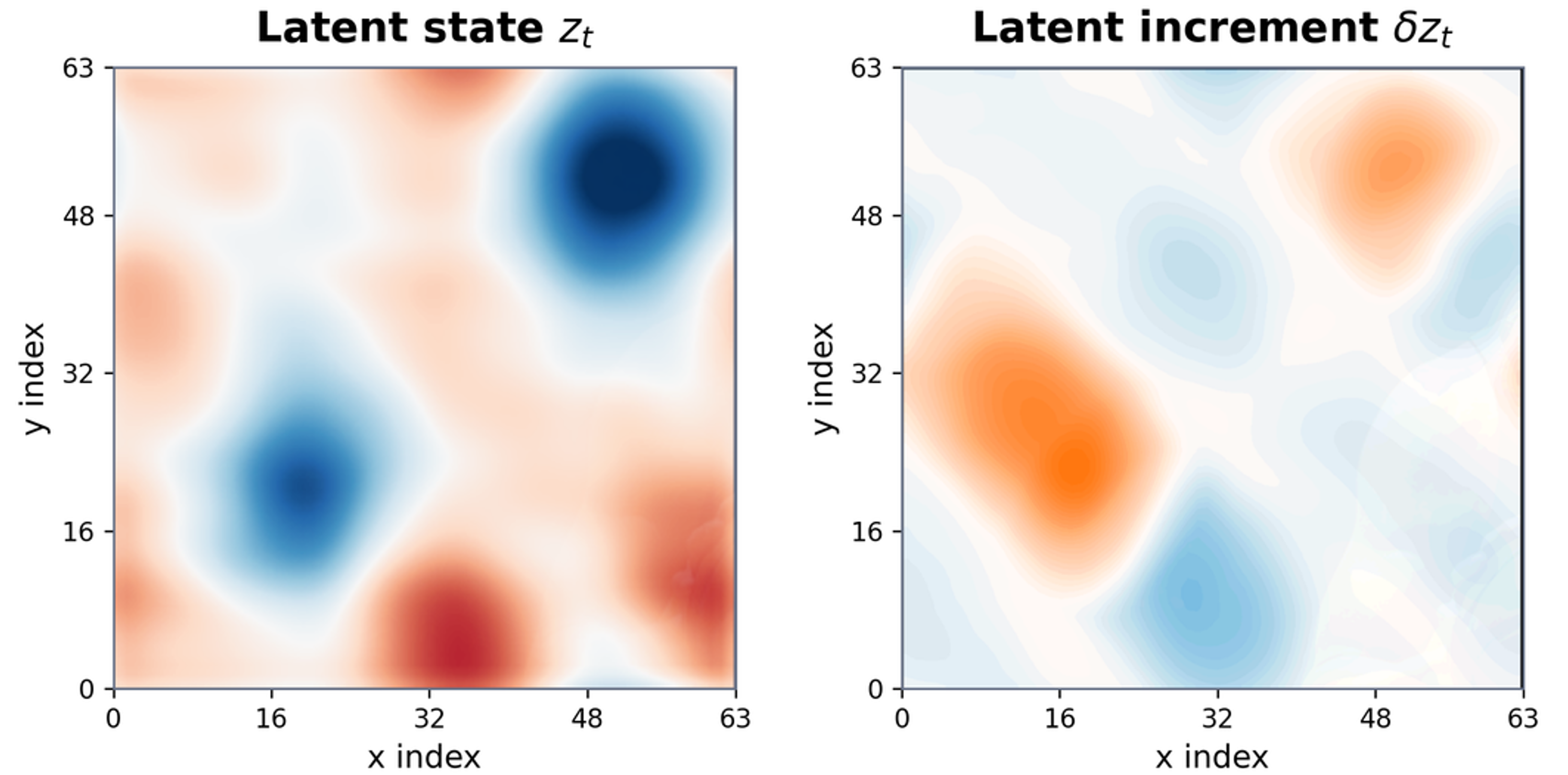}
    \caption{First principal-component projections of \(z_t\) and
    \(\delta z_t\).
    PCA is applied independently along the latent channel dimension,
    so the two panels compare dominant spatial patterns rather than
    cross-panel magnitudes.}
    \label{fig:latent_pca}
\end{figure}

\paragraph{Latent Increment Dynamics}
To examine the roles of latent states and latent increments, we
perform a forward pass using a trained GeoIncNO model on a held-out
2D Navier--Stokes (NS) test sample.
We visualize the latent state \(z_t\) and the corresponding latent
increment \(\delta z_t\) in Figure~\ref{fig:latent_pca}.
Since both are multi-channel latent fields, we apply principal
component analysis (PCA) along the latent channel dimension and
visualize the spatial response of the first principal component.
The PCA projections are computed independently, so the two panels
compare dominant spatial patterns rather than cross-panel magnitudes.

As shown in Figure~\ref{fig:latent_pca}, \(z_t\) exhibits broad
spatial organization, indicating that it preserves the current field
configuration in the latent space.
In contrast, \(\delta z_t\) presents a distinct dominant response
pattern.
This difference suggests that the latent state and latent increment
encode different information rather than repeating the same spatial
structure.
Explicitly modeling the latent increment is therefore more targeted
than imposing structure only on the full latent state, which may
entangle the current field configuration with the transition
information required for subsequent propagation.

\begin{figure}[t]
    \centering
    \includegraphics[width=0.92\linewidth]{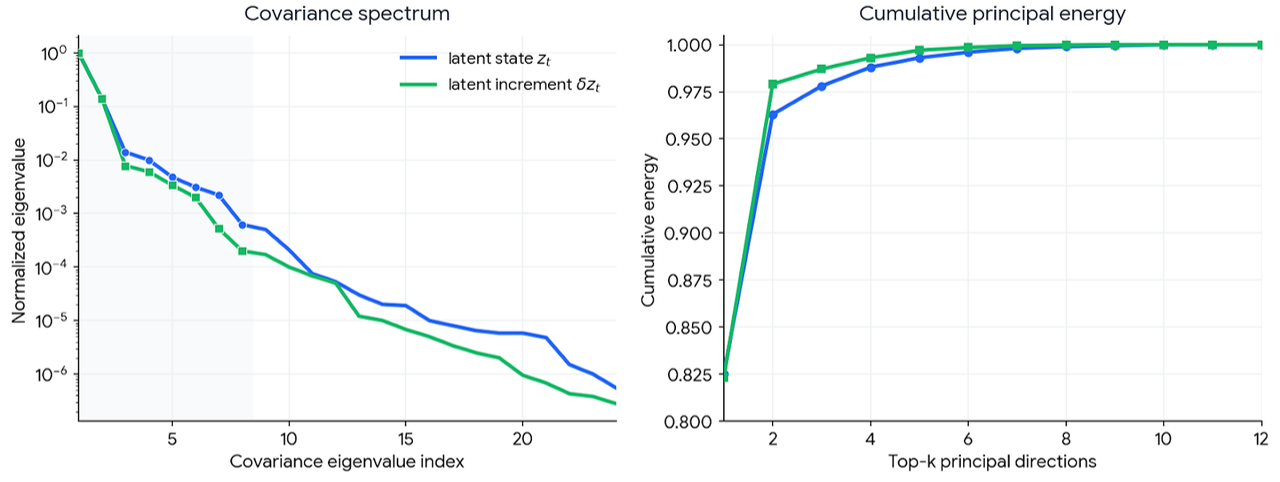}
    \caption{Covariance spectral analysis of \(z_t\) and
    \(\delta z_t\), showing faster eigenvalue decay and higher
    principal-energy concentration for latent increments.}
    \label{fig:covariance_geometry}
\end{figure}

\begin{figure}[t]
    \centering
    \includegraphics[width=0.93\linewidth]{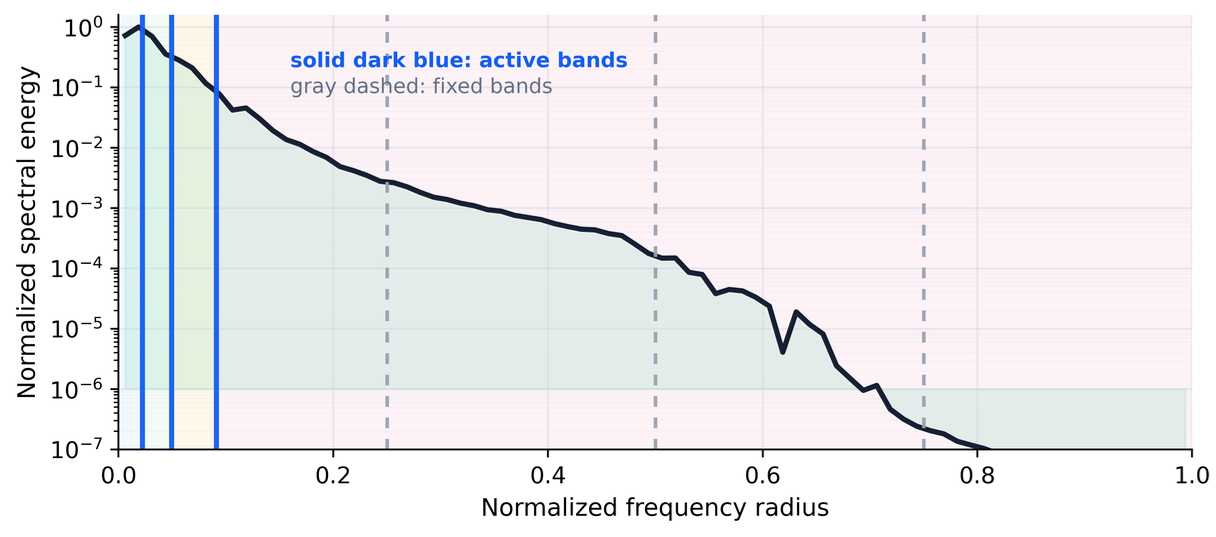}
    \caption{Radial spectral energy distribution of \(\delta z_t\).}
    \label{fig:active_bands}
\end{figure}

\paragraph{Covariance Geometry of Latent Increments}
We analyze the covariance geometry of latent increments by applying
a trained GeoIncNO model to held-out 2D NS test samples and
collecting \(z_t\) and \(\delta z_t\) during inference.
We compute their covariance matrices across the latent channels and
compare the resulting spectra in
Figure~\ref{fig:covariance_geometry}.
Specifically, we perform eigenvalue decomposition and visualize the
normalized eigenvalue spectra and cumulative energy captured by the
top-\(k\) principal directions.
As shown in Figure~\ref{fig:covariance_geometry}, the covariance
eigenvalues of \(\delta z_t\) decay more rapidly than those of
\(z_t\), and its cumulative principal energy rises faster within the
first few directions.
This indicates that the variation of \(\delta z_t\) is concentrated
in fewer principal directions, implying a more compact covariance
structure.
Therefore, its dominant transition directions can be captured by a
small number of principal modes, supporting the increment-centered
geometric design of GeoIncNO.

\paragraph{Non-uniform Increment Spectra}
Next, we analyze the spectral distribution of latent increments.
We use a trained FNO checkpoint on the 2D shallow-water equation (SWE) benchmark to verify that the
observed spectral non-uniformity is not specific to our model.
For each test sample, two consecutive physical contexts and their
coordinates \((x,y)\) are passed through the same FNO lifting layer
and GELU activation to obtain \(z_t\) and \(z_{t+1}\).
The latent increment is then computed as
\(\delta z_t = z_{t+1} - z_t\).
We apply a 2D real FFT to \(\delta z_t\) over the spatial dimensions
and average the spectral power over samples, time steps, and latent
channels.
For each Fourier mode, let \(k_x\) and \(k_y\) denote the frequency
indices along the two spatial axes, and define its radial frequency as $\rho(k_x,k_y)=\sqrt{k_x^2+k_y^2}$. 
As illustrated in Figure~\ref{fig:active_bands}, the black curve
shows the normalized radial spectral energy of \(\delta z_t\),
obtained by grouping Fourier modes according to
\(\rho(k_x,k_y)\).
The energy is non-uniformly distributed and is concentrated in
low-frequency regions and a few effective frequency intervals,
decaying rapidly as the radial frequency increases.
The gray dashed lines denote a fixed uniform frequency partition,
which assigns several bands to low-energy or nearly inactive regions.
In contrast, the blue solid lines show active-band boundaries that
adapt to the observed spectral support of \(\delta z_t\).
This comparison motivates constructing active bands according to
the increment spectral energy, so that geometric modeling focuses on
frequency regions carrying meaningful transition dynamics.

\begin{figure*}
\centering
\includegraphics[width=0.87\linewidth]{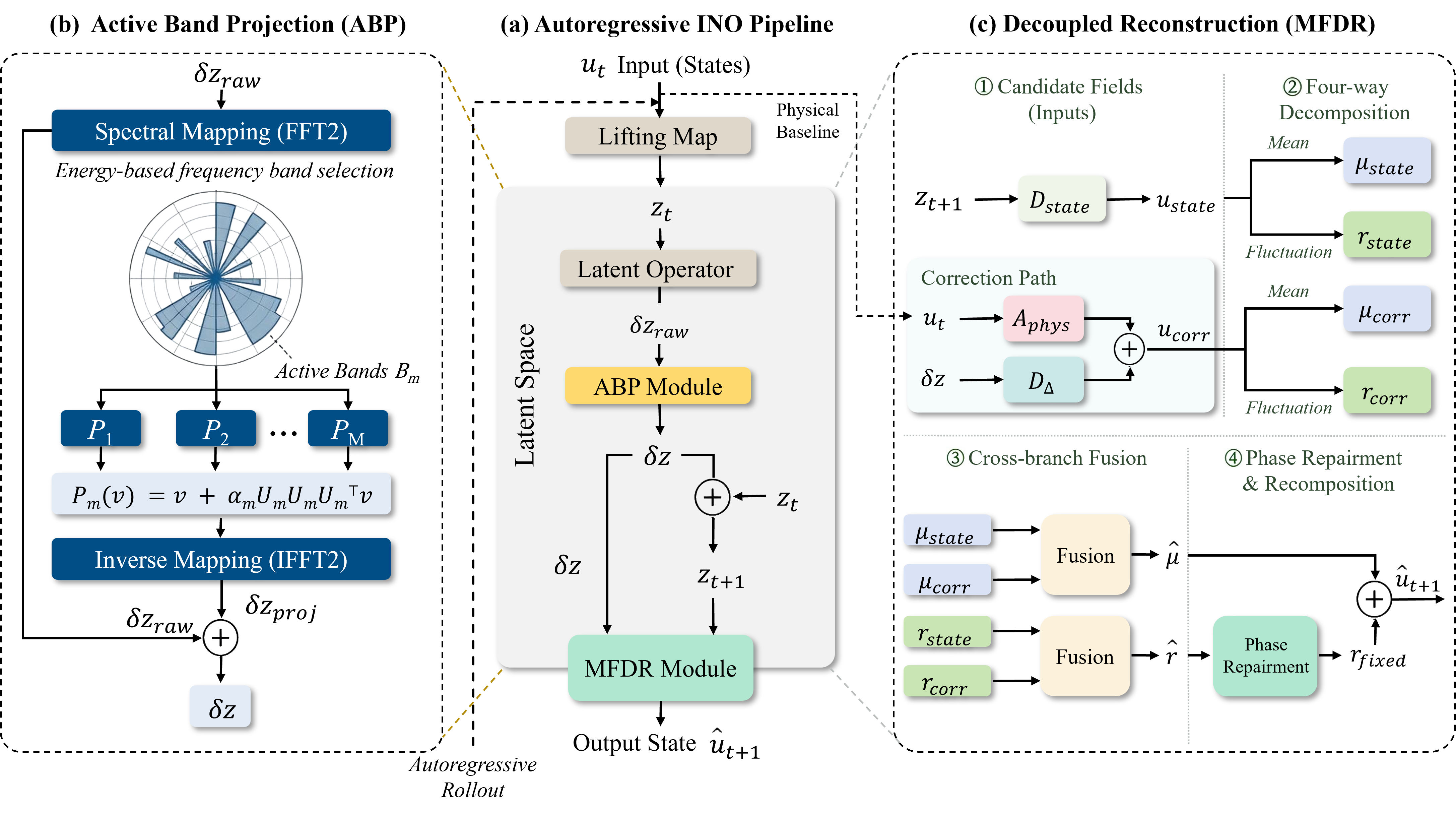}
\caption{
Overview of GeoIncNO.
ABP shapes the spectral--channel geometry of the latent increment
before residual state advancement, while MFDR performs dual-branch
mean--fluctuation reconstruction with fluctuation-only phase correction.
}
\label{fig:overview}
\end{figure*}

\section{Method}
\label{sec:method}
GeoIncNO is built on a guiding principle: it explicitly treats the latent transition increment as the object that governs autoregressive evolution. 
As shown in Figure~\ref{fig:overview}, GeoIncNO follows a geometry-aware incremental pipeline: i) the input physical field is lifted into a latent state, from which a raw latent transition increment is predicted; ii) active spectral bands are constructed from the increment, and low-rank band-wise projection is applied to shape its spectral-channel geometry; iii) the refined increment advances the latent state through residual propagation; iv) two complementary physical predictions are reconstructed from the updated state and the refined increment; and v) temporal mean and zero-mean fluctuation components are decoupled
and fused separately, after which phase correction is applied only to
the fluctuation component.

\subsection{Problem Formulation}

We consider long-horizon prediction for time-dependent PDE systems.
Let \(u_{\mathrm{in}}\) denote the observed physical fields over an
input time window, where
\(u_{\mathrm{in}} \in \mathbb{R}^{B \times T_{\mathrm{in}}
\times N^{\mathrm{in}}_1 \times \cdots
\times N^{\mathrm{in}}_{d_s} \times C_{\mathrm{in}}}\).
Here, \(B\) is the batch size, \(T_{\mathrm{in}}\) is the number of
input time steps, \(d_s\) is the spatial dimension,
\(N^{\mathrm{in}}_1,\ldots,N^{\mathrm{in}}_{d_s}\) are the input
spatial grid sizes, and \(C_{\mathrm{in}}\) is the number of input
physical variables.
The goal is to predict the future physical fields
\(\hat{u} \in \mathbb{R}^{B \times T_{\mathrm{out}}
\times N^{\mathrm{out}}_1 \times \cdots
\times N^{\mathrm{out}}_{d_s} \times C_{\mathrm{out}}}\),
where \(T_{\mathrm{out}}\) is the prediction horizon,
\(N^{\mathrm{out}}_1,\ldots,N^{\mathrm{out}}_{d_s}\) are the output
spatial grid sizes, and \(C_{\mathrm{out}}\) is the number of output
physical variables.
Let \(u_{\mathrm{out}}\) denote the corresponding ground-truth future
fields with the same shape.
The task is to learn a neural operator
\(\mathcal{G}_{\theta}:u_{\mathrm{in}}\mapsto\hat{u}\)
that approximates the PDE solution operator over the target window.

\subsection{Latent State Encoding and Increment Prediction}

\subsubsection{Encoding}

GeoIncNO first encodes the input physical field into a grid-structured
latent representation, where subsequent dynamics are modeled through
latent increments.
Specifically, the input field is mapped as
\(z_{\mathrm{in}}=E_\phi(u_{\mathrm{in}})\), where
\(z_{\mathrm{in}} \in \mathbb{R}^{B \times d \times T_{\mathrm{in}}
\times N^{\mathrm{in}}_1 \times \cdots
\times N^{\mathrm{in}}_{d_s}}\), and \(d\) denotes the latent channel
width.
The encoder preserves the spatiotemporal grid by concatenating the
physical variables with normalized temporal and spatial coordinates,
followed by a latent projection and several dimension-specific spatial
convolutional layers.
The resulting \(z_{\mathrm{in}}\) serves as the latent representation
from which GeoIncNO predicts the transition direction.

\subsubsection{Increment Prediction}

For notational simplicity, we denote \(z_{\mathrm{in}}\) by \(z_t\) when
describing a single latent transition, where \(t\) indexes the
transition itself rather than the intra-window temporal coordinate
\(\tau\).
Given \(z_t\), GeoIncNO extracts transition features as
\(h_t=\mathcal{T}_\theta(z_t)\).
The backbone \(\mathcal{T}_\theta\) consists of mixed latent blocks that
combine spectral-domain modeling, local convolution, pointwise channel
mixing, and a channel MLP, allowing the model to capture both global
modes and local spatiotemporal patterns.
An increment prediction head \(G_\psi\) then predicts the raw latent
increment as \(\delta z_{\mathrm{raw}}=G_\psi(h_t)\).
The raw increment provides an unconstrained estimate of the latent
transition direction and is subsequently refined through the
active-band mechanism before latent-state advancement.

\subsection{Active-Band Increment Projection}

The predicted increment \(\delta z_{\mathrm{raw}}\) encodes the
transition direction of the latent dynamics.
Compared with the full latent state \(z_t\), which also contains static
context, boundary-related information, and sample-specific background
structures, the increment more directly captures the dynamic change
that drives state evolution.
GeoIncNO therefore applies geometric shaping to the increment rather
than to the full state.
Since the freely predicted increment may contain unstructured spectral
components or redundant channel couplings, this shaping is performed
within active spectral bands before residual state advancement.

\subsubsection{Active-Band Construction}

To focus geometric shaping on dynamically relevant frequency regions,
GeoIncNO constructs active bands according to the spectral energy
distribution of the raw latent increment.
The construction is defined on a general \(d_s\)-dimensional spatial
domain and therefore applies to one-, two-, and three-dimensional PDE
systems.
Given
\(\delta z_{\mathrm{raw}}\in
\mathbb{R}^{B \times d \times T_{\mathrm{in}}
\times N^{\mathrm{in}}_1 \times \cdots
\times N^{\mathrm{in}}_{d_s}}\),
where \(N^{\mathrm{in}}_1,\ldots,N^{\mathrm{in}}_{d_s}\) denote the
spatial resolutions of the latent grid, we
compute its spectral representation as
\(\widehat{\delta z}_{\mathrm{raw}}
=\mathcal{F}_{\mathbf{x}}(\delta z_{\mathrm{raw}})\), where
\(\mathbf{x}=(x_1,\ldots,x_{d_s})\).
For periodic domains, \(\mathcal{F}_{\mathbf{x}}\) is implemented as a
\(d_s\)-dimensional fast Fourier transform.
For non-periodic discretizations, it can be replaced with a
boundary-compatible orthogonal spectral transform.
Let
\(\mathbf{k}=(k_1,\ldots,k_{d_s})\in\mathbb{Z}^{d_s}\)
denote a spatial frequency index.
The average spectral energy of the raw latent increment at
\(\mathbf{k}\) is defined as
\begin{align}
E(\mathbf{k})
&=
\frac{1}{B T_{\mathrm{in}} d}
\sum_{b,t,c}
\left|
\widehat{\delta z}_{\mathrm{raw}}
(b,c,t,\mathbf{k})
\right|^2.
\end{align}
The active-band partition is constructed within the effective spectral
support
\(\Omega_{\mathrm{eff}}\subseteq\mathbb{Z}^{d_s}\),
preventing inactive frequency locations from affecting the band
boundaries.
To obtain a dimension-independent radial partition, we define
\begin{align}
\rho(\mathbf{k})
&=
\left(
\sum_{j=1}^{d_s}
\left(
\frac{|k_j|}{k_j^{\max}+\epsilon}
\right)^2
\right)^{1/2},
\end{align}
where \(k_j^{\max}\) denotes the maximum retained frequency index along
the \(j\)-th spatial direction and \(\epsilon\) is a numerical
stability constant.

GeoIncNO determines the active-band boundaries from the cumulative
increment spectral energy.
The cumulative energy function over \(\Omega_{\mathrm{eff}}\) is
defined as
\begin{align}
C(r)
&=
\frac{
\sum_{\substack{
\mathbf{k}\in\Omega_{\mathrm{eff}}\\
\rho(\mathbf{k})\le r
}}
E(\mathbf{k})
}{
\sum_{\mathbf{k}\in\Omega_{\mathrm{eff}}}
E(\mathbf{k})
+\epsilon
}.
\end{align}
Given \(M\) active bands, the boundary of the \(m\)-th band is
determined by the energy quantile
\begin{align}
e_m
&=
\inf
\left\{
r \mid C(r)\ge\frac{m}{M}
\right\},
\qquad
m=1,\ldots,M.
\end{align}
The active band is $\mathcal{B}_m=
\left\{
\mathbf{k}\in\Omega_{\mathrm{eff}}
\mid
e_{m-1}\le\rho(\mathbf{k})<e_m
\right\}$.
The last band includes the right endpoint so that all frequency
locations in \(\Omega_{\mathrm{eff}}\) are covered.
For 1D PDEs, the active bands are intervals on the frequency axis.
For 2D PDEs, they correspond to radial annuli in the frequency plane,
and for 3D PDEs, they correspond to radial shells in the frequency
volume.
Thus, ABP constructs active bands in a unified \(d_s\)-dimensional
frequency space rather than relying on a specifically two-dimensional
formulation.

To maintain stable band semantics, GeoIncNO calibrates the active-band
boundaries during the early training stage.
Let \(E_b(\mathbf{k})\) denote the increment spectral energy estimated
from the \(b\)-th mini-batch.
The training-set estimate is
\(\bar{E}(\mathbf{k})
=N_{\mathrm{batch}}^{-1}
\sum_{b=1}^{N_{\mathrm{batch}}}E_b(\mathbf{k})\).
The cumulative energy function is then computed by replacing
\(E(\mathbf{k})\) with \(\bar{E}(\mathbf{k})\), and the boundaries
\(\{e_m\}_{m=1}^{M}\) are obtained using the same energy-quantile rule.
Once calibrated, the boundaries remain fixed during subsequent
training and inference.
If the effective spectral energy is insufficient or too few frequency
locations are available, GeoIncNO falls back to a uniform radial
partition.

\subsubsection{Low-Rank Band Projection}

After constructing the active bands, GeoIncNO applies a lightweight
low-rank projection within each band to reshape the channel coupling of
the spectral increment.
Here, projection refers to residual low-rank geometric shaping rather
than a strict orthogonal projection.
For a frequency location
\(\mathbf{k}\in\mathcal{B}_m\), let
\(v=a+ib\in\mathbb{C}^d\), where \(a,b\in\mathbb{R}^d\), denote the
complex channel vector of the spectral increment.

For each band, we introduce a learnable low-rank channel basis
\(U_m\in\mathbb{R}^{d\times r}\), where \(r\ll d\), and define
\begin{align}
P_m^{\mathbb{R}}(x)
&=
x+\alpha_m U_mU_m^\top x,
\end{align}
where \(\alpha_m\) is a learnable projection strength and
\(x\in\mathbb{R}^d\).
The low-rank basis restricts band-wise channel modulation to a few
learnable directions, avoiding unconstrained full-channel mixing.
The same real-valued projector is applied to the real and imaginary
parts as
\(P_m(v)=P_m^{\mathbb{R}}(a)+iP_m^{\mathbb{R}}(b)\).
This shared projector preserves the complex spectral representation
while restricting channel transformation to a low-rank subspace.

The projector \(P_m\) is applied independently to the channel vector
at each batch index, time step, and frequency location in
\(\mathcal{B}_m\).
The projected bands are placed back into their corresponding spectral
locations to form
\(\widehat{\delta z}_{\mathrm{proj}}\), and the projected increment is
obtained as
\(\delta z_{\mathrm{proj}}
=\mathcal{F}_{\mathbf{x}}^{-1}
(\widehat{\delta z}_{\mathrm{proj}})\).
We denote the complete active-band projection operator by
\(P_{\mathrm{AB}}\), such that
\(\delta z_{\mathrm{proj}}
=P_{\mathrm{AB}}(\delta z_{\mathrm{raw}})\).
The projected increment is used as a residual geometric correction:
\begin{align}
\delta z
&=
\delta z_{\mathrm{raw}}
+
\eta
\left(
P_{\mathrm{AB}}(\delta z_{\mathrm{raw}})
-
\delta z_{\mathrm{raw}}
\right), \\
z_{t+1}
&=
z_t+\delta z.
\end{align}
Here, \(\eta\) controls the correction strength.
Since this correction is applied before residual state advancement,
the active-band geometry directly participates in the latent
transition rather than serving only as an auxiliary regularizer.
The frequency-dependent latent-channel geometry underlying this
band-wise design is empirically analyzed in
Appendix~\ref{sec:app_band_geometry}.

\subsection{Dual-Branch Physical-Space Reconstruction}

After the latent update, GeoIncNO obtains the refined increment
\(\delta z\) and the updated latent state \(z_{t+1}\).
Since \(\delta z\) captures the transition-induced correction while
\(z_{t+1}\) represents the evolved latent state, GeoIncNO reconstructs
the physical output through two complementary branches.
The input-anchored correction branch preserves stable input structures
with increment-induced correction, whereas the state-decoding branch
predicts the output from the updated latent state.

\subsubsection{Input-Anchored Correction Branch}

The input-anchored correction branch maps the input physical field to
the target output window and grid using
\(u_{\mathrm{base}}=A_{\mathrm{phys}}(u_{\mathrm{in}})\), where
\(A_{\mathrm{phys}}=R_{\mathrm{sp}}\circ M_T\circ M_C\).
Here,
\(M_C:\mathbb{R}^{C_{\mathrm{in}}}\rightarrow
\mathbb{R}^{C_{\mathrm{out}}}\)
maps physical variables,
\(M_T:\mathbb{R}^{T_{\mathrm{in}}}\rightarrow
\mathbb{R}^{T_{\mathrm{out}}}\)
maps the temporal window, and \(R_{\mathrm{sp}}\) resamples the spatial grid from
\((N^{\mathrm{in}}_1,\ldots,N^{\mathrm{in}}_{d_s})\) to
\((N^{\mathrm{out}}_1,\ldots,N^{\mathrm{out}}_{d_s})\).
The refined increment is decoded into a physical correction as
\(\delta u=D_\Delta(\delta z)\), where \(D_\Delta\) consists of a
pointwise spatiotemporal convolutional head and a temporal linear
projection.
The correction branch is then given by
\begin{align}
u_{\mathrm{corr}}
&=
\alpha u_{\mathrm{base}}+\delta u,
\end{align}
where \(\alpha\) is a learnable anchor mixing coefficient.

\subsubsection{State-Decoding Branch}

The state-decoding branch directly predicts
\(u_{\mathrm{state}}=D_{\mathrm{state}}(z_{t+1})\), where
\(D_{\mathrm{state}}\) is independent of \(D_\Delta\).
While \(D_\Delta(\delta z)\) predicts a correction relative to the
input-anchored base field, \(D_{\mathrm{state}}(z_{t+1})\) predicts the
complete physical field from the evolved latent representation.
Thus, GeoIncNO obtains two complementary predictions,
\(u_{\mathrm{corr}}\) and \(u_{\mathrm{state}}\).
Both decoders map to the physical output space, i.e.,
\(D_\Delta(\delta z),\, D_{\mathrm{state}}(z_{t+1})\in
\mathbb{R}^{B\times T_{\mathrm{out}}\times N_1^{\mathrm{out}}
\times\cdots\times N_{d_s}^{\mathrm{out}}\times C_{\mathrm{out}}}\).

\subsection{Mean--Fluctuation Decoupled Reconstruction}

Directly fusing \(u_{\mathrm{corr}}\) and \(u_{\mathrm{state}}\)
mixes mean structure, fluctuation dynamics, and phase errors in the
same prediction space.
Motivated by the distinct temporal roles of the mean and zero-mean
fluctuation components illustrated in
Appendix~\ref{sec:app_mean_fluctuation}, GeoIncNO decomposes each
candidate prediction into these two components, fuses them separately,
and applies phase correction only to the fluctuation component.

\subsubsection{Mean--Fluctuation Decomposition and Fusion}

For an output sequence 
\(u \in
\mathbb{R}^{B \times T_{\mathrm{out}}
\times N^{\mathrm{out}}_1 \times \cdots
\times N^{\mathrm{out}}_{d_s}
\times C_{\mathrm{out}}}\),
we define its temporal mean and zero-mean fluctuation as
\begin{align}
\mu(u)
&=
\frac{1}{T_{\mathrm{out}}}
\sum_{\tau=1}^{T_{\mathrm{out}}}u_\tau, \\
r(u)
&=
u-\operatorname{repeat}_t(\mu(u)).
\end{align}
Here,
\(\mu(u)
\in
\mathbb{R}^{B \times 1
\times N^{\mathrm{out}}_1 \times \cdots
\times N^{\mathrm{out}}_{d_s}
\times C_{\mathrm{out}}}\), and
\(\operatorname{repeat}_t(\mu(u))\) copies the temporal mean along the
output time dimension.
By construction, \(\operatorname{mean}_t(r(u))=0\), and
\(u=\operatorname{repeat}_t(\mu(u))+r(u)\).

Applying this decomposition to the two candidate predictions gives
\begin{align}
\mu_{\mathrm{corr}}
&=
\mu(u_{\mathrm{corr}}), &
r_{\mathrm{corr}}
&=
u_{\mathrm{corr}}
-\operatorname{repeat}_t(\mu_{\mathrm{corr}}), \\
\mu_{\mathrm{state}}
&=
\mu(u_{\mathrm{state}}), &
r_{\mathrm{state}}
&=
u_{\mathrm{state}}
-\operatorname{repeat}_t(\mu_{\mathrm{state}}).
\end{align}
GeoIncNO uses separate gates for mean-field and fluctuation-field
fusion:
\begin{align}
\hat{\mu}
&=
\Gamma_\mu(z_{t+1})\odot\mu_{\mathrm{corr}}
+
\left[1-\Gamma_\mu(z_{t+1})\right]\odot\mu_{\mathrm{state}}, \\
\hat{r}
&=
\operatorname{center}_t
\left(
\Gamma_r(z_{t+1})\odot r_{\mathrm{corr}}
+
\left[1-\Gamma_r(z_{t+1})\right]\odot r_{\mathrm{state}}
\right).
\end{align}
Here, \(\Gamma_\mu\) and \(\Gamma_r\) are generated from the updated
latent state, and \(\operatorname{center}_t(\cdot)\) removes the
temporal mean to keep \(\hat{r}\) zero-mean.
The separate gates allow the mean structure and fluctuation dynamics
to use different branch preferences.

\subsubsection{Fluctuation-Only Phase Correction}

Phase shift is meaningful for non-zero temporal fluctuations rather
than for the temporal mean.
GeoIncNO therefore applies phase correction only to the fused
fluctuation \(\hat{r}\).
Let \(\mathcal{P}_{\mathrm{phase}}\) denote a lightweight temporal
advancement operator.
The phase residual and gated correction are
\begin{align}
\Delta r_{\mathrm{phase}}
&=
\mathcal{P}_{\mathrm{phase}}(\hat{r})-\hat{r}, \\
r_{\mathrm{phase}}
&=
\operatorname{center}_t
\left(
\Gamma_{\mathrm{phase}}(z_{t+1})
\odot\Delta r_{\mathrm{phase}}
\right).
\end{align}
The centering operation keeps the phase correction zero-mean and
prevents direct perturbation of the temporal mean.

The final prediction is reconstructed as
\begin{align}
\hat{u}
&=
\operatorname{repeat}_t(\hat{\mu})
+\hat{r}
+r_{\mathrm{phase}}.
\end{align}
Thus, the reconstruction separates stable mean structures,
non-stationary fluctuations, and phase correction within the
fluctuation subspace.

\subsection{Training Objective}
\label{subsec:training-objective}

GeoIncNO is trained to jointly promote accurate physical-field
prediction, structured active-band increment geometry, and limited
deviation between the raw and geometrically shaped increments.
The primary prediction loss is
\(\mathcal{L}_{\mathrm{pred}}
=\operatorname{MSE}(\hat{u},u_{\mathrm{out}})\).

The frequency--channel structure of the latent increment is regulated
through an active-band geometry loss.
For the \(m\)-th active band \(\mathcal{B}_m\), we flatten the batch,
temporal, and frequency dimensions of
\(\widehat{\delta z}_{\mathrm{raw}}^{(m)}\) and concatenate its real and
imaginary components, yielding
\(X_m\in\mathbb{R}^{N_m\times d}\).
Its centered covariance and normalized correlation matrices are
\begin{align}
C_m
&=
\frac{1}{N_m-1}
\bar{X}_m^\top\bar{X}_m, \\
R_m
&=
D_m^{-1/2}C_mD_m^{-1/2}, \\
D_m
&=
\operatorname{diag}(C_m)+\epsilon I_d,
\end{align}
where \(\bar{X}_m\) is centered along the sample dimension.
Redundant channel coupling is penalized by
\begin{align}
\mathcal{L}_{\mathrm{corr}}
&=
\frac{1}{M}
\sum_{m=1}^{M}
\frac{
\left\|
\operatorname{offdiag}(R_m)
\right\|_F^2
}{
d(d-1)
}.
\end{align}
This term suppresses strong off-diagonal correlations and encourages
different channels to represent complementary transition directions.
A variance support term is further introduced to prevent feature
collapse:
\begin{align}
\mathcal{L}_{\mathrm{var}}
&=
\frac{1}{M}
\sum_{m=1}^{M}
\frac{1}{d}
\sum_{c=1}^{d}
\left[
\max
\left(
0,
\gamma-\sqrt{C_m(c,c)+\epsilon}
\right)
\right]^2.
\end{align}
This term maintains sufficient channel variation within each active
band.
The complete geometry loss is
\(\mathcal{L}_{\mathrm{geom}}
=\mathcal{L}_{\mathrm{corr}}+\mathcal{L}_{\mathrm{var}}\).
Projection consistency is enforced by keeping the refined increment
close to the transition direction predicted by the latent backbone:
\begin{align}
\mathcal{L}_{\mathrm{proj}}
&=
\frac{
\left\|
\delta z-\delta z_{\mathrm{raw}}
\right\|_2^2
}{
\left\|
\delta z_{\mathrm{raw}}
\right\|_2^2+\epsilon
}.
\end{align}

The final objective is
\begin{align}
\mathcal{L}
&=
\mathcal{L}_{\mathrm{pred}}
+
\lambda_{\mathrm{geom}}(s)\mathcal{L}_{\mathrm{geom}}
+
\lambda_{\mathrm{proj}}(s)\mathcal{L}_{\mathrm{proj}},
\end{align}
where \(s\) denotes the training progress.
A warm-up schedule gradually increases
\(\lambda_{\mathrm{geom}}(s)\) and
\(\lambda_{\mathrm{proj}}(s)\), allowing GeoIncNO to first learn a
reliable predictive transition and then progressively enforce
increment geometry and projection consistency.

The complete inference procedure is summarized in
Algorithm~\ref{alg:GeoIncNO}.
Appendix~\ref{sec:theory} provides a conditional interpretation of GeoIncNO.
Under bounded latent reconstruction error and Lipschitz decoding,
controlling the latent increment error also controls the physical
prediction error; under spectral-support and band-wise low-rank
assumptions, ABP admits a geometric interpretation.

\section{Experiments}
\label{sec:experiments}

\subsection{Experimental Setup}
\label{subsec:experimental-setup}

\subsubsection{Baseline Models}
\label{subsec:baselines}
We select seven baselines to cover neural-operator paradigms for PDE prediction.
DeepONet \cite{lu2021deeponet} is included as a classical operator-learning baseline based on the branch--trunk formulation. 
FNO \cite{DBLP:conf/iclr/LiKALBSA21}, UNO \cite{DBLP:journals/tmlr/RahmanRA23}, and WNO \cite{tripura2023wavelet} represent transform-domain and multi-scale neural operators, covering Fourier-based global spectral modeling, U-shaped multi-resolution operator learning, and wavelet-based localized multi-scale modeling. 
We also include PINO \cite{DBLP:journals/corr/abs-2111-03794} as a physics-informed neural operator that incorporates PDE residual constraints during training.
We further include two recent latent-space neural operators, LNO \cite{wang2024lno} and LaMO \cite{tiwari2025latent}, as the most direct baselines to GeoIncNO. 
While these methods also reduce PDE prediction to latent-space operator learning, they primarily focus on learning compact latent representations or efficient latent dynamics. 
In contrast, GeoIncNO explicitly formulates the latent transition as an increment and shapes its active spectral components before state advancement. 

\subsubsection{PDE Benchmarks}
\label{subsec:pde-benchmarks}

We evaluate GeoIncNO on six PDE benchmarks generated under a
PDEBench-style protocol: Burgers, Kuramoto--Sivashinsky (KS),
Navier--Stokes (NS), shallow-water (SW), 3D compressible Euler (CE), and Maxwell.
These benchmarks span one-, two-, and three-dimensional systems and
cover convection--diffusion, spatiotemporally chaotic, vortical,
wave-propagation, compressible-flow, and electromagnetic dynamics.
Table~\ref{tab:pde_benchmarks} summarizes their spatial domains,
grid resolutions, temporal ranges, numbers of snapshots, and numbers
of physical state variables.
In the Time Range column, the notation
\([t_{\min},t_{\max}]/N_t\) denotes the simulated time interval and
the number of temporal snapshots.
The corresponding governing equations and variable definitions are
provided in Appendix~\ref{sec:pde_details}.

\subsubsection{Evaluation Metrics}

We report metrics averaged over the test set to evaluate prediction accuracy and rollout stability.
Field-level accuracy is measured by the relative \(L_2\) error
\(\mathrm{Rel}\text{-}L_2=\|u-\hat{u}\|_2/\|u\|_2\) and MSE, where Rel-\(L_2\) evaluates scale-normalized errors and MSE measures absolute deviations.
We further report the relative Sobolev error
\(\mathrm{Rel}\text{-}H^1=\|u-\hat{u}\|_{H^1}/\|u\|_{H^1}\),
which accounts for both field and gradient discrepancies.
For spectral accuracy, we use the weighted log ratio (WLR) to measure discrepancies between predicted and reference spectral energies.
For autoregressive rollouts, we report the final-step RMSE (F-RMSE) and cumulative RMSE (C-RMSE), which measure terminal and accumulated prediction errors, respectively.
Boundary RMSE (B-RMSE) is used when boundary-region accuracy is evaluated.
We also report the relative temporal-mean drift (Mean Drift) to quantify long-term bias in the rollout trajectory.
The full definitions of WLR and Mean Drift are provided in Appendix~\ref{subsec:evaluation-metrics}.

\subsubsection{Implementation Details}
\label{subsec:implementation-details}
All experiments are implemented in PyTorch on a single NVIDIA RTX 4090 GPU (24\,GB).
For each benchmark, we use 1,000 trajectories for training and 200 independently generated trajectories for testing.
All non-physics-informed baselines are trained with the MSE prediction loss only, whereas PINO additionally uses the PDE residual loss following its original formulation, and GeoIncNO is optimized with the objective in Section~\ref{subsec:training-objective}, which augments the prediction loss with active-band increment geometry regularization and projection consistency.
All models are trained for 100 epochs using Adam~\cite{DBLP:journals/corr/KingmaB14} with an initial learning rate of \(1\times10^{-4}\) and cosine annealing, together with the pushforward strategy at a rollout length of \(T=5\) for long-horizon stability.
For GeoIncNO, we use \(M=4\) active frequency bands, a low-rank projector dimension of \(r=8\), and an ABP correction strength of \(\eta=0.05\) unless otherwise specified.
Full training and dimension-specific operator details are provided in Appendix~\ref{app:impl}.

\subsection{Overall Prediction Performance}

Table~\ref{tab:main_results} compares GeoIncNO with competitive
neural-operator baselines across six PDE benchmarks in field accuracy,
local-structure preservation, spectral fidelity, and cumulative rollout
error.
The complete results, including MSE, F-RMSE, parameter counts, and
standard deviations, are provided in
Appendix Table~\ref{tab:full_results}.
GeoIncNO achieves the best results on all four primary metrics across
all benchmarks.
Compared with the FNO backbone, GeoIncNO substantially reduces
Rel-\(L_2\), from \(3.22\text{E-}2\) to
\(1.64\text{E-}2\) on 1D KS, from \(2.65\text{E-}1\) to
\(7.54\text{E-}3\) on 2D NS, and from \(2.75\text{E-}2\) to
\(5.43\text{E-}3\) on 3D CE.
The gains are particularly evident on the chaotic 1D KS benchmark,
where several baselines exhibit severe error accumulation, whereas
GeoIncNO achieves a C-RMSE of \(2.02\text{E-}1\) and a
Rel-\(H^1\) of \(2.12\text{E-}2\).
GeoIncNO also consistently outperforms the latent-space baselines LNO
and LaMO.
On 2D NS, for example, it reduces LaMO's Rel-\(L_2\), Rel-\(H^1\),
WLR, and C-RMSE from \(1.37\text{E-}2\),
\(2.42\text{E-}2\), \(2.05\text{E-}2\), and
\(1.43\text{E-}1\) to \(7.54\text{E-}3\),
\(1.33\text{E-}2\), \(1.13\text{E-}2\), and
\(7.87\text{E-}2\), respectively.
The simultaneous improvements in Rel-\(H^1\), WLR, and C-RMSE indicate
that GeoIncNO better preserves local and spectral structures throughout
autoregressive rollout, supporting the effectiveness of explicitly
structuring the latent transition increment.

\begin{table}[t!]
\centering
\caption{
Main prediction results across six PDE benchmarks}
\label{tab:main_results}
\scriptsize
\setlength{\tabcolsep}{2.3pt}
\renewcommand{\arraystretch}{0.92}
\resizebox{\linewidth}{!}{
\begin{tabular}{llcccc}
\toprule
PDE & Method
& Rel-\(L_2\)$\downarrow$
& Rel-\(H^1\)$\downarrow$
& WLR$\downarrow$
& C-RMSE$\downarrow$ \\
\midrule

\multirow{8}{*}{\textbf{1D Burgers}}
& DeepONet & $8.64\text{E-}2$ & $1.68\text{E-}1$ & $1.82\text{E-}1$ & $7.84\text{E-}5$ \\
& FNO      & $1.29\text{E-}2$ & $4.68\text{E-}2$ & $2.68\text{E-}2$ & $1.18\text{E-}5$ \\
& UNO      & $9.74\text{E-}2$ & $1.76\text{E-}1$ & $2.16\text{E-}1$ & $8.96\text{E-}5$ \\
& WNO      & $9.86\text{E-}3$ & $6.73\text{E-}2$ & $1.86\text{E-}2$ & $6.98\text{E-}6$ \\
& PINO     & $2.59\text{E-}2$ & $5.72\text{E-}2$ & $5.76\text{E-}2$ & $2.68\text{E-}5$ \\
& LNO      & $7.68\text{E-}2$ & $1.44\text{E-}1$ & $1.52\text{E-}1$ & $6.71\text{E-}5$ \\
& LaMO     & $8.72\text{E-}3$ & $7.94\text{E-}2$ & $1.43\text{E-}2$ & $5.94\text{E-}6$ \\
& GeoIncNO & $\mathbf{4.80\text{E-}3}$ & $\mathbf{2.57\text{E-}2}$ &
$\mathbf{7.87\text{E-}3}$ & $\mathbf{3.27\text{E-}6}$ \\

\cmidrule(lr){1-6}

\multirow{8}{*}{\textbf{1D KS}}
& DeepONet & $2.98\text{E-}2$ & $3.86\text{E-}2$ & $4.08\text{E-}2$ & $3.67\text{E-}1$ \\
& FNO      & $3.22\text{E-}2$ & $5.41\text{E-}2$ & $5.04\text{E-}2$ & $4.13\text{E-}1$ \\
& UNO      & $3.58\text{E-}2$ & $5.96\text{E-}2$ & $5.51\text{E-}2$ & $4.49\text{E-}1$ \\
& WNO      & $6.22\text{E-}1$ & $8.00\text{E-}1$ & $1.35\text{E+}0$ & $7.58\text{E+}1$ \\
& PINO     & $9.29\text{E-}1$ & $9.70\text{E-}1$ & $1.15\text{E+}0$ & $1.22\text{E+}2$ \\
& LNO      & $6.38\text{E-}1$ & $6.95\text{E-}1$ & $1.18\text{E+}0$ & $7.87\text{E+}1$ \\
& LaMO     & $1.88\text{E-}1$ & $3.40\text{E-}1$ & $5.03\text{E-}1$ & $2.34\text{E+}1$ \\
& GeoIncNO & $\mathbf{1.64\text{E-}2}$ & $\mathbf{2.12\text{E-}2}$ &
$\mathbf{2.24\text{E-}2}$ & $\mathbf{2.02\text{E-}1}$ \\

\cmidrule(lr){1-6}

\multirow{8}{*}{\textbf{2D NS}}
& DeepONet & $2.41\text{E-}1$ & $2.46\text{E-}1$ & $3.87\text{E-}1$ & $3.26\text{E+}0$ \\
& FNO      & $2.65\text{E-}1$ & $6.19\text{E-}1$ & $5.01\text{E-}1$ & $3.62\text{E+}0$ \\
& UNO      & $2.87\text{E-}1$ & $6.78\text{E-}1$ & $4.77\text{E-}1$ & $3.91\text{E+}0$ \\
& WNO      & $2.11\text{E-}2$ & $3.29\text{E-}2$ & $3.46\text{E-}2$ & $2.68\text{E-}1$ \\
& PINO     & $3.04\text{E-}2$ & $4.87\text{E-}2$ & $4.81\text{E-}2$ & $3.96\text{E-}1$ \\
& LNO      & $4.13\text{E-}2$ & $5.32\text{E-}2$ & $6.81\text{E-}2$ & $5.13\text{E-}1$ \\
& LaMO     & $1.37\text{E-}2$ & $2.42\text{E-}2$ & $2.05\text{E-}2$ & $1.43\text{E-}1$ \\
& GeoIncNO & $\mathbf{7.54\text{E-}3}$ & $\mathbf{1.33\text{E-}2}$ &
$\mathbf{1.13\text{E-}2}$ & $\mathbf{7.87\text{E-}2}$ \\

\cmidrule(lr){1-6}

\multirow{8}{*}{\textbf{2D SW}}
& DeepONet & $1.35\text{E-}1$ & $2.39\text{E-}1$ & $2.80\text{E-}1$ & $1.46\text{E+}0$ \\
& FNO      & $7.43\text{E-}2$ & $1.49\text{E-}1$ & $1.32\text{E-}1$ & $6.91\text{E-}1$ \\
& UNO      & $8.06\text{E-}2$ & $1.61\text{E-}1$ & $1.46\text{E-}1$ & $7.50\text{E-}1$ \\
& WNO      & $1.73\text{E-}1$ & $4.61\text{E-}1$ & $5.53\text{E-}1$ & $1.62\text{E+}0$ \\
& PINO     & $7.20\text{E-}2$ & $1.44\text{E-}1$ & $1.25\text{E-}1$ & $6.69\text{E-}1$ \\
& LNO      & $9.94\text{E-}2$ & $1.78\text{E-}1$ & $9.94\text{E-}2$ & $9.23\text{E-}1$ \\
& LaMO     & $1.14\text{E-}1$ & $9.48\text{E-}1$ & $3.07\text{E-}1$ & $1.06\text{E+}0$ \\
& GeoIncNO & $\mathbf{3.96\text{E-}2}$ & $\mathbf{7.92\text{E-}2}$ &
$\mathbf{5.47\text{E-}2}$ & $\mathbf{3.68\text{E-}1}$ \\

\cmidrule(lr){1-6}

\multirow{8}{*}{\textbf{3D CE}}
& DeepONet & $8.94\text{E-}2$ & $6.25\text{E-}1$ & $1.28\text{E-}1$ & $8.21\text{E-}1$ \\
& FNO      & $2.75\text{E-}2$ & $1.38\text{E-}1$ & $5.04\text{E-}2$ & $2.48\text{E-}1$ \\
& UNO      & $9.87\text{E-}3$ & $5.22\text{E-}2$ & $1.68\text{E-}2$ & $8.75\text{E-}2$ \\
& WNO      & $5.94\text{E-}2$ & $3.68\text{E-}1$ & $1.05\text{E-}1$ & $5.92\text{E-}1$ \\
& PINO     & $2.43\text{E-}2$ & $1.18\text{E-}1$ & $4.49\text{E-}2$ & $2.19\text{E-}1$ \\
& LNO      & $9.74\text{E-}2$ & $6.84\text{E-}1$ & $1.62\text{E-}1$ & $9.85\text{E-}1$ \\
& LaMO     & $4.74\text{E-}2$ & $4.87\text{E-}1$ & $8.25\text{E-}2$ & $4.78\text{E-}1$ \\
& GeoIncNO & $\mathbf{5.43\text{E-}3}$ & $\mathbf{2.87\text{E-}2}$ &
$\mathbf{9.24\text{E-}3}$ & $\mathbf{4.81\text{E-}2}$ \\

\cmidrule(lr){1-6}

\multirow{8}{*}{\textbf{3D Maxwell}}
& DeepONet & $8.74\text{E-}2$ & $1.15\text{E-}1$ & $1.25\text{E-}1$ & $1.58\text{E+}0$ \\
& FNO      & $1.89\text{E-}2$ & $2.67\text{E-}2$ & $3.21\text{E-}2$ & $3.52\text{E-}1$ \\
& UNO      & $2.76\text{E-}2$ & $3.58\text{E-}2$ & $4.52\text{E-}2$ & $5.16\text{E-}1$ \\
& WNO      & $3.82\text{E-}2$ & $5.76\text{E-}2$ & $6.52\text{E-}2$ & $7.84\text{E-}1$ \\
& PINO     & $1.68\text{E-}2$ & $2.32\text{E-}2$ & $2.79\text{E-}2$ & $3.05\text{E-}1$ \\
& LNO      & $6.27\text{E-}2$ & $8.95\text{E-}2$ & $8.96\text{E-}2$ & $1.14\text{E+}0$ \\
& LaMO     & $2.31\text{E-}2$ & $4.68\text{E-}2$ & $3.94\text{E-}2$ & $4.43\text{E-}1$ \\
& GeoIncNO & $\mathbf{9.92\text{E-}3}$ & $\mathbf{1.33\text{E-}2}$ &
$\mathbf{1.47\text{E-}2}$ & $\mathbf{1.71\text{E-}1}$ \\

\bottomrule
\end{tabular}
}
\end{table}

\subsection{Long-Horizon Rollout Stability}
\label{subsec:rollout}
To evaluate the long-horizon stability of GeoIncNO, we conduct autoregressive rollout experiments across all six PDE benchmarks. Starting from the initial condition, each model is applied autoregressively, where the prediction at the current step is fed back as the input for the next step. We report MSE, Rel-\(L_2\), and Rel-\(H^1\) at five checkpoints, \(0.2T\)--\(1.0T\), to measure how errors accumulate over the horizon. All models are evaluated under the same rollout protocol, initial conditions, prediction horizon, and evaluation metrics.
As shown in Figure~\ref{fig:rollout} (in Appendix~\ref{app:rollout}), GeoIncNO attains the lowest error at every checkpoint on all three metrics across the 1D, 2D, and 3D benchmarks, and its advantage widens with the horizon. It stays nearly flat on the chaotic 1D KS benchmark where several baselines diverge, and preserves gradient-level structure on 2D NS and SW where competing operators inflate Rel-\(H^1\). Compared with the FNO backbone and the closest latent-space baseline LaMO, GeoIncNO shows a markedly slower error-accumulation slope rather than merely a lower one-step error, so the gap is largest at the late checkpoints where long-horizon stability matters most.
Overall, these results indicate that explicitly structuring the latent transition increment, rather than repeatedly applying an unconstrained state-to-state update, suppresses the high-frequency error accumulation that destabilizes autoregressive rollout.

\subsection{Ablation Study}
\label{subsec:ablation}
To examine whether each component contributes to the long-horizon
prediction behavior of GeoIncNO, we conduct component ablations on the
2D Navier--Stokes (NS) benchmark. We define the variants as follows:
i) \textbf{Base LatentNO} contains the encoder, latent backbone,
residual latent-state update, and decoder, without geometric shaping or
structured reconstruction; ii) \textbf{+ Fixed-band LIG} adds latent
increment geometry (LIG), where low-rank projectors shape the latent
increment over fixed frequency bands; iii) \textbf{+ Active Bands}
replaces fixed bands with spectrum-adaptive active bands constructed
from the increment spectral distribution; iv) \textbf{+ MFDR} adds
mean--fluctuation decoupled reconstruction, which separately fuses
temporal mean and zero-mean fluctuation components; v) \textbf{+ FFPC}
applies full-field phase correction to the reconstructed prediction;
and vi) \textbf{GeoIncNO} keeps the same components but restricts phase
correction to the zero-mean fluctuation component. Each row in
Table~\ref{tab:ablation} is cumulative with respect to the previous row
unless otherwise specified.
\begin{table}[t!]
\centering
\caption{Cumulative component ablation on the 2D NS}
\label{tab:ablation}
\small
\setlength{\tabcolsep}{4pt}
\begin{threeparttable}
\begin{tabular}{lcccc}
\toprule
Variant
& Rel-$L_2\;\downarrow$
& C-RMSE\,$\downarrow$
& WLR\,$\downarrow$
& \shortstack{Mean\\Drift}\,$\downarrow$ \\
\midrule
\midrule
Base LatentNO
& 2.18E-02 & 1.91E-01 & 2.95E-02 & 2.35E-03 \\
+Fixed-band LIG
& 1.78E-02 & 1.52E-01 & 2.40E-02 & 1.84E-03 \\
+Active Bands
& 1.32E-02 & 1.10E-01 & 1.65E-02 & 1.32E-03 \\
+MFDR
& 9.10E-03 & 8.72E-02 & 1.34E-02 & 1.05E-04 \\
+FFPC
& 8.02E-03 & 8.05E-02 & 1.20E-02 & 3.42E-04 \\
GeoIncNO
& \textbf{7.54E-03} & \textbf{7.87E-02}
& \textbf{1.13E-02} & \textbf{6.21E-05} \\
\bottomrule
\end{tabular}
\end{threeparttable}
\end{table}
Table~\ref{tab:ablation} shows monotone gains from each component. Fixed
increment geometry already improves over the unconstrained baseline
(Rel-$L_2$: 2.18E-02 to 1.78E-02), and active bands
reduce all four metrics further, supporting spectrum-adaptive shaping.
MFDR produces the largest Mean Drift drop, from 1.32E-03 to
1.05E-04, while also lowering prediction and spectral error.
Full-field phase correction lowers Rel-$L_2$ but raises Mean Drift back
to 3.42E-04, whereas restricting it to the zero-mean
fluctuation (GeoIncNO) reaches the best value on every metric and drives
Mean Drift to 6.21E-05. Overall, active-band increment
geometry and mean--fluctuation decoupling jointly improve spectral
fidelity and rollout stability while confining phase correction to the
fluctuation subspace preserves the temporal mean.

\subsection{Additional Results}
\label{subsec:additional_results}

Additional experiments further evaluate GeoIncNO beyond the standard
benchmark setting.
Cross-parameter generalization to an unseen Reynolds number is reported
in Appendix~\ref{subsec:cross_parameter}, while cross-resolution
transfer from coarse to finer grids is examined in
Appendix~\ref{subsec:cross-resolution}.
Extended rollout behavior under the unseen parameter regime is analyzed
in Appendix~\ref{sec:long}.
Results on the RealPDEBench fluid--structure interaction task are
provided in Appendix~\ref{app:realpdebench}, and qualitative comparisons
across competitive 1D, 2D, and 3D systems are shown in
Appendix~\ref{app:qualitative}.

\section{Conclusion}
\label{sec:conclusion}

We proposed GeoIncNO for stable long-horizon PDE prediction. GeoIncNO treats the latent transition increment as the object governing autoregressive evolution, rather than only improving the latent state representation or direct state-to-state mapping. This formulation allows the repeated dynamical update to be regularized before residual state advancement. Specifically, active-band low-rank projection shapes the spectral-channel geometry of latent increments, and mean--fluctuation decoupled reconstruction separates stable mean structures, dynamic fluctuations, and fluctuation-level phase correction in the physical output space. Experiments on six one-, two-, and three-dimensional PDE benchmarks have shown that GeoIncNO improves prediction accuracy, spectral fidelity, and rollout stability over competitive neural-operator baselines. Ablation, cross-parameter, cross-resolution, and qualitative results have further confirmed that structured latent increments and mean--fluctuation decoupling reduce error accumulation during autoregressive prediction.

\section{Limitations and Ethical Considerations}
GeoIncNO currently calibrates active bands from training-set spectral
statistics and keeps them fixed during inference.
Future work may explore adaptive band construction and broader
validation under more diverse geometries, boundary conditions, and
distribution shifts.
In addition, incorporating physical constraints or uncertainty
estimation could further improve reliability for very long rollouts.
This work uses no human subjects,
personal information, or privacy-sensitive data.

\section{Generative AI Usage}
Generative AI was used only for language editing.

\bibliographystyle{ACM-Reference-Format}
\bibliography{software}

\appendix

\section{Related Work}
\label{app:RL}
\subsection{Neural Operators and Latent Dynamics}
Unlike physics-informed neural networks~\cite{raissi2019pinns,karniadakis2021physics}, which embed governing equations into the loss and are typically retrained per instance, neural operators~\cite{kovachki2023neural} learn mappings between function spaces, with FNO~\cite{DBLP:conf/iclr/LiKALBSA21} and DeepONet~\cite{lu2021deeponet} as the two representative architectures. Subsequent work follows two routes.
The first route strengthens the operator backbone: multi-scale and adaptive Fourier mixing~\cite{wen2022u,guibas2022afno}, spectral-basis enrichment~\cite{gupta2021multiwavelet,hu2025wavelet,you2025mscalefno}, attention-based interactions in place of fixed spectral mixing~\cite{li2023oformer,hao2023gnot}, and Koopman-linearized Fourier operators~\cite{xiong2024koopman}.
The second route compresses PDE dynamics into a learned latent space following the encode--evolve--decode paradigm~\cite{morton2018deep,lusch2018koopman,brunton2022koopman}, realized within neural operators by LSM~\cite{wu2023lsm}, LNO~\cite{wang2024lno}, and PI-Latent-NO~\cite{karumuri2025pilatentno}.
Both routes improve prediction by enhancing backbone expressiveness, broadening spectral coverage, or learning compact latent representations. However, they mostly model the state-to-state mapping or latent-state evolution, while the latent increment that drives temporal evolution is rarely modeled as a structured object in its own right.

\section{Motivation}
\label{sec:moti}
\subsection{Mean--Fluctuation Separation}
\label{sec:app_mean_fluctuation}
The physical output contains components with different temporal properties.
As shown in Figure~\ref{fig:mean_fluctuation}, the temporal mean $\mu(u)$ captures the stable background or slowly varying structure, whereas the zero-mean fluctuation $r(u)$ represents non-stationary oscillations and phase-related dynamics.
This decomposition suggests that mean errors and fluctuation errors should not be treated identically: mean errors correspond to mean drift, while fluctuation errors are more related to oscillatory mismatch and phase misalignment.
This motivates a structured reconstruction strategy that decouples mean and fluctuation components and restricts phase correction to the zero-mean fluctuation field.
\begin{figure}[t]
    \centering
    \includegraphics[width=\linewidth]{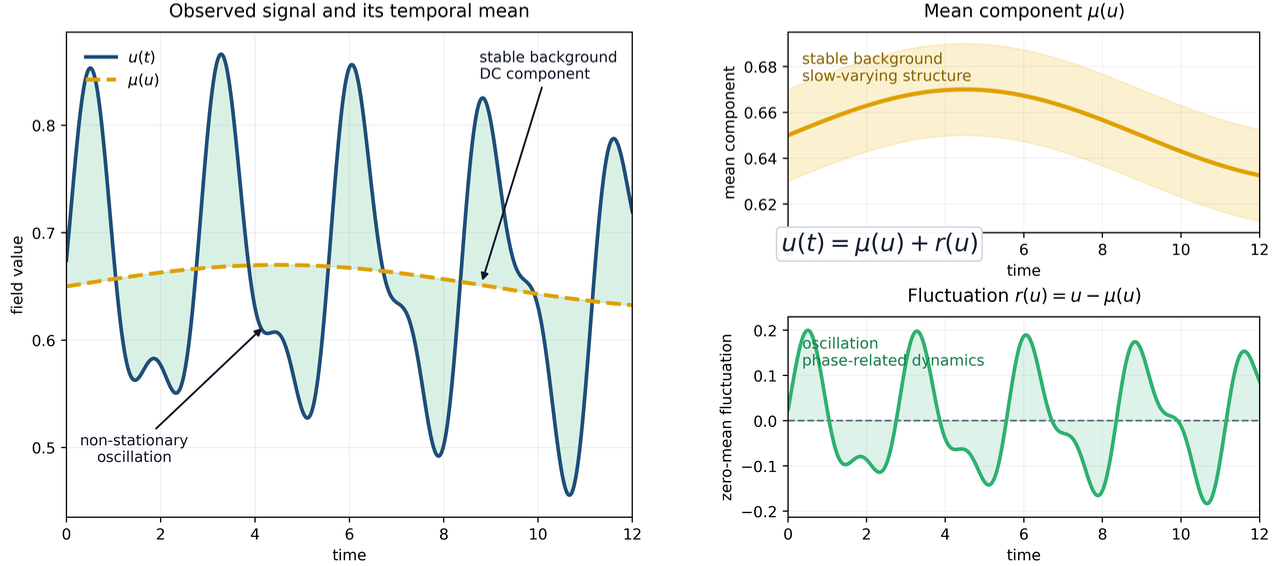}
    \caption{Mean--fluctuation decomposition of a physical time series, separating the stable mean component from the zero-mean fluctuation dynamics.}
    \label{fig:mean_fluctuation}
\end{figure}

\subsection{Frequency-Dependent Band Geometry}
\label{sec:app_band_geometry}
Beyond the non-uniform spectral support, different active frequency bands also exhibit distinct latent-channel geometries.
As shown in Figure~\ref{fig:band_geometry}, the left panel visualizes the latent channel correlation matrices within representative low-, middle-, and high-frequency active bands.
Different bands present clearly different channel-coupling patterns, indicating that the geometric organization of latent increments varies across frequency regions.
The right panel further reports the effective rank and top eigen energy of different active bands.
These statistics differ noticeably across bands, suggesting that the complexity and concentration of channel-wise variation are frequency dependent.
A single global geometric constraint is therefore unlikely to characterize all frequency regions adequately.
This supports a band-wise design in which each active band uses an independent low-rank geometric projector, allowing incremental geometry to match local spectral and channel structures.

\begin{figure}[t]
    \centering
    \includegraphics[width=\linewidth]{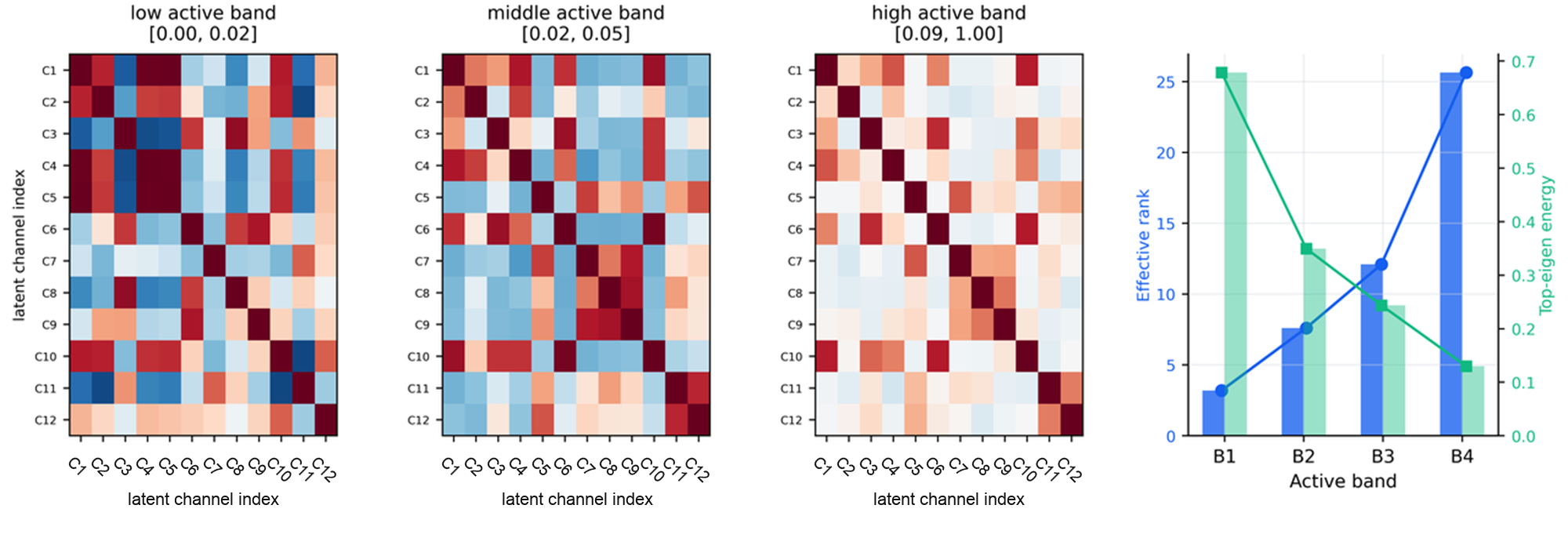}
    \caption{Frequency-dependent channel geometry of latent increments. Distinct channel-correlation matrices and band-dependent rank statistics indicate that a single global geometry is insufficient for all active bands.}
    \label{fig:band_geometry}
\end{figure}

\section{Method}

\subsection{Theoretical Analysis}
\label{sec:theory}

This section provides a conditional theoretical interpretation of GeoIncNO. The goal is not to prove that the learned model always outperforms existing neural operators, but to clarify how increment-centered modeling and active-band geometric shaping can reduce physical prediction error under explicit structural assumptions.

\subsubsection{Preliminaries and Notation}

Let $\Omega \subset \mathbb{R}^{d_s}$ denote the spatial domain, where
$d_s$ is the spatial dimension. In this work, $d_s=1,2,3$ corresponds
to the one-, two-, and three-dimensional PDE benchmarks considered in
the experiments. Let $I_{\mathrm{in}}$ and $I_{\mathrm{out}}$ denote
the input and output time windows, respectively. We consider input and
output physical fields
$u_{\mathrm{in}}\in\mathcal{X}_{\mathrm{in}}
:=L^2(I_{\mathrm{in}}\times\Omega;\mathbb{R}^{C_{\mathrm{in}}})$
and
$u_{\mathrm{out}}\in\mathcal{X}_{\mathrm{out}}
:=L^2(I_{\mathrm{out}}\times\Omega;\mathbb{R}^{C_{\mathrm{out}}})$,
where $C_{\mathrm{in}}$ and $C_{\mathrm{out}}$ are the numbers of input
and output physical variables. Unless otherwise specified, $\|\cdot\|$
denotes the $L^2$ norm in the corresponding function space, and its
discrete Euclidean or Frobenius counterpart on finite grids.

Let \(\mathcal{G}_\tau:\mathcal{A}\subset\mathcal{X}_{\mathrm{in}}\rightarrow\mathcal{X}_{\mathrm{out}}\) denote the ground-truth PDE solution operator over a prediction span \(\tau>0\), where \(\mathcal{A}\) is the admissible set of input trajectories. 
The target output is \(u_{\mathrm{out}}^\star=\mathcal{G}_\tau(u_{\mathrm{in}})\).

To compare input and output fields in a common latent space, we introduce two compatible encoding maps \(E_{\mathrm{in}}:\mathcal{X}_{\mathrm{in}}\rightarrow\mathcal{Z}\) and \(E_{\mathrm{out}}:\mathcal{X}_{\mathrm{out}}\rightarrow\mathcal{Z}\), where \(\mathcal{Z}\) is a latent function space. 
Here, \(E_{\mathrm{in}}\) corresponds to the encoder \(E_\phi\) used in the main architecture, while \(E_{\mathrm{out}}\) is introduced only for analysis to embed the ground-truth output into the same latent space. 
The current latent state and the ground-truth next latent state are defined as \(z_t=E_{\mathrm{in}}(u_{\mathrm{in}})\) and \(z_{t+1}^\star=E_{\mathrm{out}}(\mathcal{G}_\tau(u_{\mathrm{in}}))\). 
The ground-truth latent increment is then
\begin{equation}
    \Delta_\tau^\star(u_{\mathrm{in}})
    :=
    z_{t+1}^\star-z_t
    =
    E_{\mathrm{out}}(\mathcal{G}_\tau(u_{\mathrm{in}}))
    -
    E_{\mathrm{in}}(u_{\mathrm{in}}).
\end{equation}

GeoIncNO learns a latent increment operator \(\Delta_\tau^{\mathrm{pred}}\) and predicts the next latent state by residual advancement, \(z_{t+1}^{\mathrm{pred}}=z_t+\Delta_\tau^{\mathrm{pred}}(z_t)\). 
The corresponding physical-space prediction is written as \(\widehat{\mathcal{G}}_\tau(u_{\mathrm{in}})=D(z_{t+1}^{\mathrm{pred}})\), where \(D:\mathcal{Z}\rightarrow\mathcal{X}_{\mathrm{out}}\) denotes the overall decoding map. 
In the actual GeoIncNO architecture, \(D\) is implemented by the increment decoder, the state decoder, and the mean-fluctuation reconstruction module; in the theoretical analysis, we treat them as a unified decoder when deriving error bounds.

For a time-dependent field
$v\in L^2(I\times\Omega;\mathbb{R}^C)$, we define the temporal mean and zero-mean fluctuation operators as
$\mu_t(v)(x):=\frac{1}{|I|}\int_I v(t,x)\,dt$ and
$\operatorname{center}_t(v)(t,x):=v(t,x)-\mu_t(v)(x)$. By construction,
$\mu_t(\operatorname{center}_t(v))=0$. This continuous definition
corresponds to the discrete temporal mean $\mu(\cdot)$ and centering operation used in the main method. We use $\mathcal{F}_{\mathbf{x}}$ to
denote the spatial spectral transform over all $d_s$ spatial dimensions, with frequency variable
$\xi\in\mathbb{R}^{d_s}$ in the continuous setting or
$\mathbf{k}\in\mathbb{Z}^{d_s}$ on a discrete grid. We use
$\mathcal{F}_t$ to denote the temporal Fourier transform, with temporal frequency variable $\omega\in\mathbb{R}$. A hat over a variable denotes its spectral representation unless it is explicitly used for a model prediction, such as $\hat{u}$.

\subsubsection{Increment-Centered Latent Evolution}

A standard latent-space neural operator maps the current latent state directly to the next latent state, i.e., \(u_{\mathrm{in}}\xrightarrow{E_{\mathrm{in}}}z_t\xrightarrow{\mathcal{K}_{\theta}}z_{t+1}\xrightarrow{D}\hat{u}\), where \(\mathcal{K}_{\theta}:\mathcal{Z}\rightarrow\mathcal{Z}\). 
Such a formulation focuses on learning a state-to-state transition. 
GeoIncNO instead adopts an increment-centered formulation, \(z_{t+1}^{\mathrm{pred}}=z_t+\Delta_\tau^{\mathrm{pred}}(z_t)\), which separates the current latent state from the direction that drives its evolution.

The increment formulation is consistent with the integral form of an evolutionary PDE. 
Suppose that a physical trajectory satisfies \(\partial_t u(t)=\mathcal{N}(u(t))\). 
Then the exact time advancement over a span \(\tau\) satisfies
\begin{equation}
    u(t+\tau)-u(t)
    =
    \int_t^{t+\tau}\mathcal{N}(u(s))\,ds
    =
    \tau\overline{\mathcal{N}}_{[t,t+\tau]},
\end{equation}
where \(\overline{\mathcal{N}}_{[t,t+\tau]}:=\frac{1}{\tau}\int_t^{t+\tau}\mathcal{N}(u(s))\,ds\) is the average evolution direction over \([t,t+\tau]\). 
Thus, learning the future state can be equivalently viewed as learning the state increment induced by the underlying dynamics.

The same interpretation applies in the latent space. 
If the latent trajectory \(z(t)\) is differentiable, then
\begin{equation}
    \Delta z^\star(t,\tau)
    :=
    z(t+\tau)-z(t)
    =
    \int_t^{t+\tau}\partial_s z(s)\,ds.
\end{equation}
If \(\partial_t z\) is locally Lipschitz in time, then \(\Delta z^\star(t,\tau)=\tau\partial_t z(t)+O(\tau^2)\), showing that the latent increment represents the average local evolution direction up to higher-order terms.

GeoIncNO parameterizes this latent increment through a raw increment prediction followed by active-band geometric correction:
\begin{align}
    \delta z_{\mathrm{raw}}
    &=
    G_\psi(\mathcal{T}_\theta(z_t)), \\
    \Delta_\tau^{\mathrm{pred}}(z_t)
    &=
    \delta z_{\mathrm{raw}}
    +
    \eta\left(P_{\mathrm{AB}}(\delta z_{\mathrm{raw}})-\delta z_{\mathrm{raw}}\right), \\
    z_{t+1}^{\mathrm{pred}}
    &=
    z_t+\Delta_\tau^{\mathrm{pred}}(z_t).
\end{align}
In this form, \(G_\psi(\mathcal{T}_\theta(\cdot))\) provides an unconstrained transition direction, while \(P_{\mathrm{AB}}\) supplies a structured active-band correction. 
Therefore, the geometric constraint acts on the latent increment that drives state evolution, rather than on the full latent state itself.

\subsubsection{Common Assumptions}
\label{subsec:assumptions}

The following assumptions are used to provide conditional theoretical interpretations of GeoIncNO. 
They are not intended to claim that any trained instance of GeoIncNO is guaranteed to outperform existing neural operators. 
Instead, they specify the structural conditions under which the proposed increment-centered and active-band designs can be related to prediction error control.

\textbf{A0. Function spaces and spectral representation.}
We assume that the spatial domain $\Omega\subset\mathbb{R}^{d_s}$ is either periodic or admits an orthogonal spectral basis compatible with the boundary conditions. For periodic domains,
$\mathcal{F}_{\mathbf{x}}$ is the $d_s$-dimensional Fourier transform.
For non-periodic domains, the Fourier basis can be replaced by a
Laplacian eigenbasis, a Chebyshev basis, or another orthogonal basis adapted to the boundary conditions. All spectral decompositions, orthogonality arguments, and applications of Parseval's identity below are understood with respect to the chosen basis.

\textbf{A1. Latent reconstruction and decoder regularity.}
We assume that the latent representation preserves the information of the ground-truth solution trajectories up to a bounded reconstruction error. 
Specifically, for the output-side encoder \(E_{\mathrm{out}}\) and the unified decoder \(D\), there exists \(\varepsilon_{\mathrm{lat}}\ge0\) such that
\begin{equation}
    \sup_{u_{\mathrm{out}}\in \mathcal{G}_\tau(\mathcal{A})}
    \left\|
    D(E_{\mathrm{out}}(u_{\mathrm{out}}))-u_{\mathrm{out}}
    \right\|_{\mathcal{X}_{\mathrm{out}}}
    \le
    \varepsilon_{\mathrm{lat}} .
\end{equation}
We also assume that \(D\) is Lipschitz continuous on the relevant latent trajectory set; that is, there exists \(L_D>0\) such that
\begin{equation}
    \left\|
    D(z_1)-D(z_2)
    \right\|_{\mathcal{X}_{\mathrm{out}}}
    \le
    L_D
    \left\|
    z_1-z_2
    \right\|_{\mathcal{Z}} .
\end{equation}
This assumption allows latent increment errors to be transferred to physical-space prediction errors.

\textbf{A2. Active-band calibration.}
In GeoIncNO, active bands are calibrated from the spectral energy distribution of the model-predicted raw increment \(\delta z_{\mathrm{raw}}\), whereas the theoretical analysis concerns the spectral structure of the ground-truth latent increment \(\Delta_\tau^\star\). 
We assume that the calibrated effective spectral support \(\Omega_{\mathrm{eff}}\) covers the dominant spectral support of \(\Delta_\tau^\star\) up to a calibration error \(\varepsilon_{\mathrm{cal}}\):
\begin{equation}
    \left\|
    \mathcal{F}_{\mathbf{x}}(\Delta_\tau^\star)
    \right\|_{\Omega_{\mathrm{eff}}^c}^2
    \le
    \varepsilon_{\mathrm{cal}}^2 .
\end{equation}
Here, \(\mathcal{F}_{\mathbf{x}}(\Delta_\tau^\star)\) denotes the spectral representation of the ground-truth latent increment, and \(\|\cdot\|_{\Omega_{\mathrm{eff}}^c}\) denotes the spectral \(L^2\) norm restricted to the complement of \(\Omega_{\mathrm{eff}}\).

\textbf{A3. Approximate low-rank structure within active bands.}
Let \(\{\mathcal{B}_m\}_{m=1}^{M}\) be the active bands, which are mutually disjoint and cover \(\Omega_{\mathrm{eff}}\). 
For each band \(\mathcal{B}_m\), we assume that there exists a low-dimensional subspace \(\mathcal{U}_m\subset\mathbb{R}^d\) with orthogonal projector \(Q_m\) such that the ground-truth latent increment has small residual energy outside these subspaces:
\begin{equation}
    \sum_{m=1}^{M}
    \left\|
    (I-Q_m)\mathcal{F}_{\mathbf{x}}(\Delta_{\tau,m}^\star)
    \right\|^2
    \le
    \varepsilon_{\mathrm{band}}^2 .
\end{equation}
Here, \(\Delta_{\tau,m}^\star\) denotes the restriction of \(\Delta_\tau^\star\) to the \(m\)-th active band in the spectral domain. 
This assumption specifies the spectral condition under which low-rank active-band shaping has a theoretical interpretation.

\textbf{A4. Controlled out-of-subspace and out-of-band prediction energy.}
Approximate low-rank structure of the ground-truth increment alone is insufficient to control the model error, since the learned increment may introduce additional energy outside the active low-rank subspaces or outside the effective spectral support. 
We therefore assume that the predicted latent increment \(\Delta_\tau^{\mathrm{pred}}\) satisfies
\begin{equation}
\begin{aligned}
    \sum_{m=1}^{M}
    \left\|
    (I-Q_m)\mathcal{F}_{\mathbf{x}}(\Delta_{\tau,m}^{\mathrm{pred}})
    \right\|^2
    &\le
    \varepsilon_{\mathrm{pred,out}}^2, \\
    \left\|
    \mathcal{F}_{\mathbf{x}}(\Delta_{\tau}^{\mathrm{pred}})
    \right\|_{\Omega_{\mathrm{eff}}^c}^2
    &\le
    \varepsilon_{\mathrm{outband}}^2 .
\end{aligned}
\end{equation}
Here, \(\Delta_{\tau,m}^{\mathrm{pred}}\) denotes the restriction of the predicted latent increment to \(\mathcal{B}_m\) in the spectral domain. 
The quantities \(\varepsilon_{\mathrm{pred, out}}\) and \(\varepsilon_{\mathrm{outband}}\) control the prediction energy outside the active low-rank subspaces and outside the effective spectral support, respectively.
This assumption corresponds to the role of the increment geometry regularizer and the projection consistency loss.

\subsubsection{Latent Increment Error Controls Physical Prediction Error}
\label{subsec:latent-error-control}

We first show that the physical-space prediction error can be controlled by the latent increment error, provided that the decoder is Lipschitz continuous and the latent representation has bounded output-side reconstruction error. 
This result justifies why GeoIncNO focuses on organizing the latent increment.

\begin{theorem}[Latent increment error control]
\label{thm:increment-error}
Under Assumption A1, for any \(u_{\mathrm{in}}\in\mathcal{A}\), the one-step physical prediction error of GeoIncNO satisfies
\begin{equation}
    \left\|
    \widehat{\mathcal{G}}_\tau(u_{\mathrm{in}})
    -
    \mathcal{G}_\tau(u_{\mathrm{in}})
    \right\|_{\mathcal{X}_{\mathrm{out}}}
    \le
    \varepsilon_{\mathrm{lat}}
    +
    L_D
    \left\|
    \Delta_\tau^{\mathrm{pred}}(z_t)
    -
    \Delta_\tau^\star(u_{\mathrm{in}})
    \right\|_{\mathcal{Z}},
\end{equation}
where \(z_t=E_{\mathrm{in}}(u_{\mathrm{in}})\), \(\Delta_\tau^{\mathrm{pred}}(z_t)\) is the predicted latent increment, and \(\Delta_\tau^\star(u_{\mathrm{in}})\) is the ground-truth latent increment.
\end{theorem}

\begin{proof}
Let \(z_t=E_{\mathrm{in}}(u_{\mathrm{in}})\) and \(z_{t+1}^\star=E_{\mathrm{out}}(\mathcal{G}_\tau(u_{\mathrm{in}}))\). 
The predicted and ground-truth next latent states satisfy \(z_{t+1}^{\mathrm{pred}}=z_t+\Delta_\tau^{\mathrm{pred}}(z_t)\) and \(z_{t+1}^\star=z_t+\Delta_\tau^\star(u_{\mathrm{in}})\), respectively. 
Using the unified decoder \(D\), the model prediction is \(\widehat{\mathcal{G}}_\tau(u_{\mathrm{in}})=D(z_{t+1}^{\mathrm{pred}})\).

By adding and subtracting \(D(z_{t+1}^\star)\), we obtain
\begin{align}
    &
    \left\|
    \widehat{\mathcal{G}}_\tau(u_{\mathrm{in}})
    -
    \mathcal{G}_\tau(u_{\mathrm{in}})
    \right\|_{\mathcal{X}_{\mathrm{out}}} \nonumber \\
    &\le
    \left\|
    D(z_{t+1}^{\mathrm{pred}})
    -
    D(z_{t+1}^\star)
    \right\|_{\mathcal{X}_{\mathrm{out}}}
    +
    \left\|
    D(z_{t+1}^\star)
    -
    \mathcal{G}_\tau(u_{\mathrm{in}})
    \right\|_{\mathcal{X}_{\mathrm{out}}}.
\end{align}
By the Lipschitz continuity of \(D\), the first term is bounded by \(L_D\|z_{t+1}^{\mathrm{pred}}-z_{t+1}^\star\|_{\mathcal{Z}}\). 
Since \(z_{t+1}^{\mathrm{pred}}-z_{t+1}^\star=\Delta_\tau^{\mathrm{pred}}(z_t)-\Delta_\tau^\star(u_{\mathrm{in}})\), we have
\begin{equation}
    \left\|
    D(z_{t+1}^{\mathrm{pred}})
    -
    D(z_{t+1}^\star)
    \right\|_{\mathcal{X}_{\mathrm{out}}}
    \le
    L_D
    \left\|
    \Delta_\tau^{\mathrm{pred}}(z_t)
    -
    \Delta_\tau^\star(u_{\mathrm{in}})
    \right\|_{\mathcal{Z}}.
\end{equation}

The second term is bounded by the output-side reconstruction error in Assumption A1:
\begin{equation}
\begin{aligned}
    &
    \left\|
    D(z_{t+1}^\star)
    -
    \mathcal{G}_\tau(u_{\mathrm{in}})
    \right\|_{\mathcal{X}_{\mathrm{out}}} \\
    &=
    \left\|
    D\!\left(E_{\mathrm{out}}(\mathcal{G}_\tau(u_{\mathrm{in}}))\right)
    -
    \mathcal{G}_\tau(u_{\mathrm{in}})
    \right\|_{\mathcal{X}_{\mathrm{out}}} \\
    &\le
    \varepsilon_{\mathrm{lat}}.
\end{aligned}
\end{equation}
Combining the two bounds proves the theorem.
\end{proof}

Algorithm~\ref{alg:GeoIncNO} summarizes the complete inference procedure of GeoIncNO.
Given the input physical fields, the model predicts and geometrically refines a latent increment, advances the latent state, and reconstructs the future fields through dual-branch MFDR with fluctuation-only phase correction.
\begin{algorithm}[t]
    \caption{The inference processing of GeoIncNO}
    \label{alg:GeoIncNO}
    \begin{algorithmic}[1]
        \Require Input physical fields \(u_{\mathrm{in}}\), ABP operator \(P_{\mathrm{AB}}\), correction strength \(\eta\)
        \Ensure Predicted future physical fields \(\hat{u}\)
        
        \State \textbf{Step 1: Latent Increment Prediction}
        \State Encode the input field \(z_{\mathrm{in}} = E_\phi(u_{\mathrm{in}})\) and set \(z_t=z_{\mathrm{in}}\)
        \State Extract transition features \(h_t=\mathcal{T}_\theta(z_t)\)
        \State Predict the raw latent increment \(\delta z_{\mathrm{raw}}=G_\psi(h_t)\)
        
        \State \textbf{Step 2: ABP Correction}
        \State Apply ABP, \(\delta z_{\mathrm{proj}}=P_{\mathrm{AB}}(\delta z_{\mathrm{raw}})\)
        \State Compute the refined increment \(\delta z
            =
            \delta z_{\mathrm{raw}}
            +
            \eta(\delta z_{\mathrm{proj}}-\delta z_{\mathrm{raw}})\)
     
        \State Update the latent state \(z_{t+1}=z_t+\delta z\)
        
        \State \textbf{Step 3: Dual-Branch Reconstruction}
        \State Compute the input-anchored prediction \(u_{\mathrm{corr}}=\alpha A_{\mathrm{phys}}(u_{\mathrm{in}})+D_\Delta(\delta z)\)
        \State Compute the state-decoded prediction \(u_{\mathrm{state}}=D_{\mathrm{state}}(z_{t+1})\)
        
        \State \textbf{Step 4: MFDR Fusion}
        \State Decompose \(u_{\mathrm{corr}}\) and \(u_{\mathrm{state}}\) into mean and fluctuation components
        \State Fuse the mean components with \(\Gamma_\mu\) to obtain \(\hat{\mu}\)
        \State Fuse and center the fluctuation components with \(\Gamma_r\) to obtain \(\hat{r}\)
        
        \State \textbf{Step 5: Fluctuation-Only Phase Correction}
        \State Apply phase correction to the fused fluctuation and obtain \(r_{\mathrm{phase}}\)
        \State Reconstruct the final prediction
        \(
            \hat{u}
            =
            \operatorname{repeat}_t(\hat{\mu})
            +
            \hat{r}
            +
            r_{\mathrm{phase}}
        \)
        
        \State \Return \(\hat{u}\)
    \end{algorithmic}
\end{algorithm}

\section{Experiments}
\subsection{Experiments Setup}
\subsubsection{PDE Benchmarks}
\label{sec:pde_details}

\begin{table}[t!]
\centering
\caption{Summary of PDE benchmarks.}
\label{tab:pde_benchmarks}
\resizebox{\linewidth}{!}{
\begin{tabular}{llllll}
\toprule
Benchmark & Dimension & Domain & Resolution & Time Range & Variables \\
\midrule
Burgers & 1D & $[0,2\pi)$ & $1024$ & $[0,1.0]/100$ & 1 \\
KS & 1D & $[0,64]$ & $1024$ & $[0,10.0]/100$ & 1 \\
NS & 2D & $[0,1)^2$ & $256^2$ & $[0,20]/21$ & 1 \\
SW & 2D & $[0,2\pi)^2$ & $256^2$ & $[0,2.0]/21$ & 3 \\
CE & 3D & $[0,1]^3$ & $32^3$ & $[0,1.0]/20$ & 5 \\
Maxwell & 3D & $[0,1]^3$ & $32^3$ & $[0,1.0]/20$ & 6 \\
\bottomrule
\end{tabular}
}
\end{table}

\begin{table*}[t!]
\centering
\caption{Prediction accuracy and parameter count across six PDE benchmarks.
Results are reported as mean $\pm$ standard deviation over five independent runs; lower values indicate better accuracy.}
\label{tab:full_results}
\begin{threeparttable}
\resizebox{\textwidth}{!}{
\begin{tabular}{llccccccc}
\toprule
PDE & Method
& Rel-\(L_2\)$\downarrow$
& Rel-\(H^1\)$\downarrow$
& MSE$\downarrow$
& WLR$\downarrow$
& F-RMSE$\downarrow$
& C-RMSE$\downarrow$
& Params (M) \\
\midrule
\midrule

\multirow{8}{*}{\textbf{1D Burgers}}
& DeepONet & $8.64\text{E-}2$\,{\scriptsize$\pm\,8.58\text{E-}3$} & $1.68\text{E-}1$\,{\scriptsize$\pm\,7.92\text{E-}3$} & $2.61\text{E-}3$\,{\scriptsize$\pm\,1.79\text{E-}4$} & $1.82\text{E-}1$\,{\scriptsize$\pm\,1.16\text{E-}2$} & $1.59\text{E-}6$\,{\scriptsize$\pm\,1.71\text{E-}7$} & $7.84\text{E-}5$\,{\scriptsize$\pm\,8.04\text{E-}6$} & $0.45$ \\
& FNO      & $1.29\text{E-}2$\,{\scriptsize$\pm\,1.56\text{E-}3$} & $4.68\text{E-}2$\,{\scriptsize$\pm\,2.45\text{E-}3$} & $5.82\text{E-}5$\,{\scriptsize$\pm\,4.71\text{E-}6$} & $2.68\text{E-}2$\,{\scriptsize$\pm\,1.27\text{E-}3$} & $2.39\text{E-}7$\,{\scriptsize$\pm\,1.52\text{E-}8$} & $1.18\text{E-}5$\,{\scriptsize$\pm\,1.04\text{E-}6$} & $0.45$ \\
& UNO      & $9.74\text{E-}2$\,{\scriptsize$\pm\,4.60\text{E-}3$} & $1.76\text{E-}1$\,{\scriptsize$\pm\,1.09\text{E-}2$} & $3.32\text{E-}3$\,{\scriptsize$\pm\,3.33\text{E-}4$} & $2.16\text{E-}1$\,{\scriptsize$\pm\,1.97\text{E-}2$} & $1.87\text{E-}6$\,{\scriptsize$\pm\,1.19\text{E-}7$} & $8.96\text{E-}5$\,{\scriptsize$\pm\,8.52\text{E-}6$} & $0.45$ \\
& WNO      & $9.86\text{E-}3$\,{\scriptsize$\pm\,1.12\text{E-}3$} & $6.73\text{E-}2$\,{\scriptsize$\pm\,3.07\text{E-}3$} & $3.40\text{E-}5$\,{\scriptsize$\pm\,3.86\text{E-}6$} & $1.86\text{E-}2$\,{\scriptsize$\pm\,1.94\text{E-}3$} & $1.29\text{E-}7$\,{\scriptsize$\pm\,9.54\text{E-}9$} & $6.98\text{E-}6$\,{\scriptsize$\pm\,4.06\text{E-}7$} & $0.45$ \\
& PINO     & $2.59\text{E-}2$\,{\scriptsize$\pm\,3.27\text{E-}3$} & $5.72\text{E-}2$\,{\scriptsize$\pm\,4.21\text{E-}3$} & $2.35\text{E-}4$\,{\scriptsize$\pm\,1.24\text{E-}5$} & $5.76\text{E-}2$\,{\scriptsize$\pm\,3.07\text{E-}3$} & $5.27\text{E-}7$\,{\scriptsize$\pm\,6.17\text{E-}8$} & $2.68\text{E-}5$\,{\scriptsize$\pm\,2.58\text{E-}6$} & $0.45$ \\
& LNO      & $7.68\text{E-}2$\,{\scriptsize$\pm\,8.72\text{E-}3$} & $1.44\text{E-}1$\,{\scriptsize$\pm\,1.54\text{E-}2$} & $2.06\text{E-}3$\,{\scriptsize$\pm\,1.87\text{E-}4$} & $1.52\text{E-}1$\,{\scriptsize$\pm\,1.94\text{E-}2$} & $1.33\text{E-}6$\,{\scriptsize$\pm\,1.03\text{E-}7$} & $6.71\text{E-}5$\,{\scriptsize$\pm\,6.17\text{E-}6$} & $0.47$ \\
& LaMO     & $8.72\text{E-}3$\,{\scriptsize$\pm\,1.01\text{E-}3$} & $7.94\text{E-}2$\,{\scriptsize$\pm\,7.75\text{E-}3$} & $2.66\text{E-}5$\,{\scriptsize$\pm\,3.15\text{E-}6$} & $1.43\text{E-}2$\,{\scriptsize$\pm\,1.35\text{E-}3$} & $1.05\text{E-}7$\,{\scriptsize$\pm\,1.10\text{E-}8$} & $5.94\text{E-}6$\,{\scriptsize$\pm\,2.90\text{E-}7$} & $0.45$ \\
& GeoIncNO & $\mathbf{4.80\text{E-}3}$\,{\scriptsize$\pm\,1.75\text{E-}4$} & $\mathbf{2.57\text{E-}2}$\,{\scriptsize$\pm\,9.95\text{E-}4$} & $\mathbf{8.06\text{E-}6}$\,{\scriptsize$\pm\,2.50\text{E-}7$} & $\mathbf{7.87\text{E-}3}$\,{\scriptsize$\pm\,2.88\text{E-}4$} & $\mathbf{5.78\text{E-}8}$\,{\scriptsize$\pm\,1.83\text{E-}9$} & $\mathbf{3.27\text{E-}6}$\,{\scriptsize$\pm\,1.25\text{E-}7$} & $0.45$ \\

\cmidrule(lr){1-9}

\multirow{8}{*}{\textbf{1D KS}}
& DeepONet & $2.98\text{E-}2$\,{\scriptsize$\pm\,2.95\text{E-}3$} & $3.86\text{E-}2$\,{\scriptsize$\pm\,2.93\text{E-}3$} & $6.22\text{E-}3$\,{\scriptsize$\pm\,4.75\text{E-}4$} & $4.08\text{E-}2$\,{\scriptsize$\pm\,2.56\text{E-}3$} & $4.29\text{E-}2$\,{\scriptsize$\pm\,2.90\text{E-}3$} & $3.67\text{E-}1$\,{\scriptsize$\pm\,4.57\text{E-}2$} & $0.45$ \\
& FNO      & $3.22\text{E-}2$\,{\scriptsize$\pm\,3.22\text{E-}3$} & $5.41\text{E-}2$\,{\scriptsize$\pm\,5.24\text{E-}3$} & $7.26\text{E-}3$\,{\scriptsize$\pm\,4.32\text{E-}4$} & $5.04\text{E-}2$\,{\scriptsize$\pm\,5.39\text{E-}3$} & $5.18\text{E-}2$\,{\scriptsize$\pm\,3.05\text{E-}3$} & $4.13\text{E-}1$\,{\scriptsize$\pm\,3.19\text{E-}2$} & $0.45$ \\
& UNO      & $3.58\text{E-}2$\,{\scriptsize$\pm\,4.62\text{E-}3$} & $5.96\text{E-}2$\,{\scriptsize$\pm\,5.92\text{E-}3$} & $8.97\text{E-}3$\,{\scriptsize$\pm\,8.28\text{E-}4$} & $5.51\text{E-}2$\,{\scriptsize$\pm\,5.69\text{E-}3$} & $5.74\text{E-}2$\,{\scriptsize$\pm\,6.70\text{E-}3$} & $4.49\text{E-}1$\,{\scriptsize$\pm\,4.98\text{E-}2$} & $0.45$ \\
& WNO      & $6.22\text{E-}1$\,{\scriptsize$\pm\,4.01\text{E-}2$} & $8.00\text{E-}1$\,{\scriptsize$\pm\,3.82\text{E-}2$} & $2.71\text{E+}0$\,{\scriptsize$\pm\,1.94\text{E-}1$} & $1.35\text{E+}0$\,{\scriptsize$\pm\,9.15\text{E-}2$} & $1.77\text{E+}0$\,{\scriptsize$\pm\,1.11\text{E-}1$} & $7.58\text{E+}1$\,{\scriptsize$\pm\,9.49\text{E+}0$} & $0.45$ \\
& PINO     & $9.29\text{E-}1$\,{\scriptsize$\pm\,1.11\text{E-}1$} & $9.70\text{E-}1$\,{\scriptsize$\pm\,6.96\text{E-}2$} & $6.04\text{E+}0$\,{\scriptsize$\pm\,6.08\text{E-}1$} & $1.15\text{E+}0$\,{\scriptsize$\pm\,9.04\text{E-}2$} & $1.51\text{E+}0$\,{\scriptsize$\pm\,1.85\text{E-}1$} & $1.22\text{E+}2$\,{\scriptsize$\pm\,1.02\text{E+}1$} & $0.45$ \\
& LNO      & $6.38\text{E-}1$\,{\scriptsize$\pm\,4.31\text{E-}2$} & $6.95\text{E-}1$\,{\scriptsize$\pm\,4.58\text{E-}2$} & $2.85\text{E+}0$\,{\scriptsize$\pm\,2.64\text{E-}1$} & $1.18\text{E+}0$\,{\scriptsize$\pm\,7.95\text{E-}2$} & $1.56\text{E+}0$\,{\scriptsize$\pm\,1.48\text{E-}1$} & $7.87\text{E+}1$\,{\scriptsize$\pm\,9.55\text{E+}0$} & $0.47$ \\
& LaMO     & $1.88\text{E-}1$\,{\scriptsize$\pm\,1.48\text{E-}2$} & $3.40\text{E-}1$\,{\scriptsize$\pm\,2.16\text{E-}2$} & $2.47\text{E-}1$\,{\scriptsize$\pm\,3.21\text{E-}2$} & $5.03\text{E-}1$\,{\scriptsize$\pm\,4.44\text{E-}2$} & $6.74\text{E-}1$\,{\scriptsize$\pm\,3.55\text{E-}2$} & $2.34\text{E+}1$\,{\scriptsize$\pm\,1.15\text{E+}0$} & $0.45$ \\
& GeoIncNO & $\mathbf{1.64\text{E-}2}$\,{\scriptsize$\pm\,5.26\text{E-}4$} & $\mathbf{2.12\text{E-}2}$\,{\scriptsize$\pm\,1.09\text{E-}3$} & $\mathbf{1.88\text{E-}3}$\,{\scriptsize$\pm\,1.08\text{E-}4$} & $\mathbf{2.24\text{E-}2}$\,{\scriptsize$\pm\,9.77\text{E-}4$} & $\mathbf{2.36\text{E-}2}$\,{\scriptsize$\pm\,7.16\text{E-}4$} & $\mathbf{2.02\text{E-}1}$\,{\scriptsize$\pm\,8.51\text{E-}3$} & $0.45$ \\

\cmidrule(lr){1-9}

\multirow{8}{*}{\textbf{2D NS}}
& DeepONet & $2.41\text{E-}1$\,{\scriptsize$\pm\,3.13\text{E-}2$} & $2.46\text{E-}1$\,{\scriptsize$\pm\,2.21\text{E-}2$} & $5.69\text{E-}2$\,{\scriptsize$\pm\,7.26\text{E-}3$} & $3.87\text{E-}1$\,{\scriptsize$\pm\,4.57\text{E-}2$} & $3.89\text{E-}1$\,{\scriptsize$\pm\,1.79\text{E-}2$} & $3.26\text{E+}0$\,{\scriptsize$\pm\,3.46\text{E-}1$} & $9.98$ \\
& FNO      & $2.65\text{E-}1$\,{\scriptsize$\pm\,2.73\text{E-}2$} & $6.19\text{E-}1$\,{\scriptsize$\pm\,5.61\text{E-}2$} & $6.88\text{E-}2$\,{\scriptsize$\pm\,4.66\text{E-}3$} & $5.01\text{E-}1$\,{\scriptsize$\pm\,4.98\text{E-}2$} & $5.06\text{E-}1$\,{\scriptsize$\pm\,2.76\text{E-}2$} & $3.62\text{E+}0$\,{\scriptsize$\pm\,2.97\text{E-}1$} & $9.49$ \\
& UNO      & $2.87\text{E-}1$\,{\scriptsize$\pm\,2.40\text{E-}2$} & $6.78\text{E-}1$\,{\scriptsize$\pm\,8.55\text{E-}2$} & $8.07\text{E-}2$\,{\scriptsize$\pm\,9.64\text{E-}3$} & $4.77\text{E-}1$\,{\scriptsize$\pm\,3.21\text{E-}2$} & $4.80\text{E-}1$\,{\scriptsize$\pm\,4.20\text{E-}2$} & $3.91\text{E+}0$\,{\scriptsize$\pm\,2.35\text{E-}1$} & $9.38$ \\
& WNO      & $2.11\text{E-}2$\,{\scriptsize$\pm\,2.59\text{E-}3$} & $3.29\text{E-}2$\,{\scriptsize$\pm\,3.91\text{E-}3$} & $4.36\text{E-}4$\,{\scriptsize$\pm\,3.07\text{E-}5$} & $3.46\text{E-}2$\,{\scriptsize$\pm\,3.44\text{E-}3$} & $3.16\text{E-}2$\,{\scriptsize$\pm\,3.06\text{E-}3$} & $2.68\text{E-}1$\,{\scriptsize$\pm\,1.55\text{E-}2$} & $9.48$ \\
& PINO     & $3.04\text{E-}2$\,{\scriptsize$\pm\,3.34\text{E-}3$} & $4.87\text{E-}2$\,{\scriptsize$\pm\,4.42\text{E-}3$} & $9.06\text{E-}4$\,{\scriptsize$\pm\,1.01\text{E-}4$} & $4.81\text{E-}2$\,{\scriptsize$\pm\,4.33\text{E-}3$} & $4.97\text{E-}2$\,{\scriptsize$\pm\,2.24\text{E-}3$} & $3.96\text{E-}1$\,{\scriptsize$\pm\,2.87\text{E-}2$} & $9.49$ \\
& LNO      & $4.13\text{E-}2$\,{\scriptsize$\pm\,1.93\text{E-}3$} & $5.32\text{E-}2$\,{\scriptsize$\pm\,6.60\text{E-}3$} & $1.67\text{E-}3$\,{\scriptsize$\pm\,2.00\text{E-}4$} & $6.81\text{E-}2$\,{\scriptsize$\pm\,7.88\text{E-}3$} & $6.27\text{E-}2$\,{\scriptsize$\pm\,4.46\text{E-}3$} & $5.13\text{E-}1$\,{\scriptsize$\pm\,2.56\text{E-}2$} & $9.97$ \\
& LaMO     & $1.37\text{E-}2$\,{\scriptsize$\pm\,1.64\text{E-}3$} & $2.42\text{E-}2$\,{\scriptsize$\pm\,3.04\text{E-}3$} & $1.84\text{E-}4$\,{\scriptsize$\pm\,9.62\text{E-}6$} & $2.05\text{E-}2$\,{\scriptsize$\pm\,1.77\text{E-}3$} & $1.78\text{E-}2$\,{\scriptsize$\pm\,9.06\text{E-}4$} & $1.43\text{E-}1$\,{\scriptsize$\pm\,1.57\text{E-}2$} & $9.67$ \\
& GeoIncNO & $\mathbf{7.54\text{E-}3}$\,{\scriptsize$\pm\,4.25\text{E-}4$} & $\mathbf{1.33\text{E-}2}$\,{\scriptsize$\pm\,4.36\text{E-}4$} & $\mathbf{5.57\text{E-}5}$\,{\scriptsize$\pm\,2.54\text{E-}6$} & $\mathbf{1.13\text{E-}2}$\,{\scriptsize$\pm\,5.46\text{E-}4$} & $\mathbf{9.79\text{E-}3}$\,{\scriptsize$\pm\,3.70\text{E-}4$} & $\mathbf{7.87\text{E-}2}$\,{\scriptsize$\pm\,4.74\text{E-}3$} & $9.77$ \\

\cmidrule(lr){1-9}

\multirow{8}{*}{\textbf{2D SW}}
& DeepONet & $1.35\text{E-}1$\,{\scriptsize$\pm\,1.09\text{E-}2$} & $2.39\text{E-}1$\,{\scriptsize$\pm\,1.51\text{E-}2$} & $1.17\text{E-}2$\,{\scriptsize$\pm\,1.06\text{E-}3$} & $2.80\text{E-}1$\,{\scriptsize$\pm\,3.00\text{E-}2$} & $2.15\text{E-}1$\,{\scriptsize$\pm\,1.34\text{E-}2$} & $1.46\text{E+}0$\,{\scriptsize$\pm\,1.04\text{E-}1$} & $9.98$ \\
& FNO      & $7.43\text{E-}2$\,{\scriptsize$\pm\,9.63\text{E-}3$} & $1.49\text{E-}1$\,{\scriptsize$\pm\,1.49\text{E-}2$} & $3.53\text{E-}3$\,{\scriptsize$\pm\,2.91\text{E-}4$} & $1.32\text{E-}1$\,{\scriptsize$\pm\,1.17\text{E-}2$} & $7.65\text{E-}2$\,{\scriptsize$\pm\,4.23\text{E-}3$} & $6.91\text{E-}1$\,{\scriptsize$\pm\,4.43\text{E-}2$} & $9.49$ \\
& UNO      & $8.06\text{E-}2$\,{\scriptsize$\pm\,5.94\text{E-}3$} & $1.61\text{E-}1$\,{\scriptsize$\pm\,1.53\text{E-}2$} & $4.16\text{E-}3$\,{\scriptsize$\pm\,2.68\text{E-}4$} & $1.46\text{E-}1$\,{\scriptsize$\pm\,9.30\text{E-}3$} & $8.45\text{E-}2$\,{\scriptsize$\pm\,4.31\text{E-}3$} & $7.50\text{E-}1$\,{\scriptsize$\pm\,7.40\text{E-}2$} & $9.38$ \\
& WNO      & $1.73\text{E-}1$\,{\scriptsize$\pm\,1.12\text{E-}2$} & $4.61\text{E-}1$\,{\scriptsize$\pm\,5.62\text{E-}2$} & $1.92\text{E-}2$\,{\scriptsize$\pm\,2.26\text{E-}3$} & $5.53\text{E-}1$\,{\scriptsize$\pm\,2.82\text{E-}2$} & $3.22\text{E-}1$\,{\scriptsize$\pm\,2.10\text{E-}2$} & $1.62\text{E+}0$\,{\scriptsize$\pm\,1.65\text{E-}1$} & $9.48$ \\
& PINO     & $7.20\text{E-}2$\,{\scriptsize$\pm\,4.55\text{E-}3$} & $1.44\text{E-}1$\,{\scriptsize$\pm\,8.10\text{E-}3$} & $3.32\text{E-}3$\,{\scriptsize$\pm\,4.13\text{E-}4$} & $1.25\text{E-}1$\,{\scriptsize$\pm\,1.17\text{E-}2$} & $7.27\text{E-}2$\,{\scriptsize$\pm\,6.19\text{E-}3$} & $6.69\text{E-}1$\,{\scriptsize$\pm\,7.47\text{E-}2$} & $9.49$ \\
& LNO      & $9.94\text{E-}2$\,{\scriptsize$\pm\,1.13\text{E-}2$} & $1.78\text{E-}1$\,{\scriptsize$\pm\,1.09\text{E-}2$} & $6.32\text{E-}3$\,{\scriptsize$\pm\,3.37\text{E-}4$} & $9.94\text{E-}2$\,{\scriptsize$\pm\,8.11\text{E-}3$} & $5.77\text{E-}2$\,{\scriptsize$\pm\,4.67\text{E-}3$} & $9.23\text{E-}1$\,{\scriptsize$\pm\,7.82\text{E-}2$} & $9.97$ \\
& LaMO     & $1.14\text{E-}1$\,{\scriptsize$\pm\,1.22\text{E-}2$} & $9.48\text{E-}1$\,{\scriptsize$\pm\,9.69\text{E-}2$} & $8.32\text{E-}3$\,{\scriptsize$\pm\,1.07\text{E-}3$} & $3.07\text{E-}1$\,{\scriptsize$\pm\,1.64\text{E-}2$} & $1.79\text{E-}1$\,{\scriptsize$\pm\,1.42\text{E-}2$} & $1.06\text{E+}0$\,{\scriptsize$\pm\,7.83\text{E-}2$} & $9.67$ \\
& GeoIncNO & $\mathbf{3.96\text{E-}2}$\,{\scriptsize$\pm\,2.37\text{E-}3$} & $\mathbf{7.92\text{E-}2}$\,{\scriptsize$\pm\,2.95\text{E-}3$} & $\mathbf{1.00\text{E-}3}$\,{\scriptsize$\pm\,3.52\text{E-}5$} & $\mathbf{5.47\text{E-}2}$\,{\scriptsize$\pm\,2.44\text{E-}3$} & $\mathbf{3.17\text{E-}2}$\,{\scriptsize$\pm\,1.38\text{E-}3$} & $\mathbf{3.68\text{E-}1}$\,{\scriptsize$\pm\,1.41\text{E-}2$} & $9.77$ \\

\cmidrule(lr){1-9}

\multirow{8}{*}{\textbf{3D CE}}
& DeepONet & $8.94\text{E-}2$\,{\scriptsize$\pm\,5.92\text{E-}3$} & $6.25\text{E-}1$\,{\scriptsize$\pm\,7.72\text{E-}2$} & $2.80\text{E-}3$\,{\scriptsize$\pm\,2.31\text{E-}4$} & $1.28\text{E-}1$\,{\scriptsize$\pm\,1.51\text{E-}2$} & $6.82\text{E-}2$\,{\scriptsize$\pm\,6.26\text{E-}3$} & $8.21\text{E-}1$\,{\scriptsize$\pm\,4.05\text{E-}2$} & $32.93$ \\
& FNO      & $2.75\text{E-}2$\,{\scriptsize$\pm\,3.57\text{E-}3$} & $1.38\text{E-}1$\,{\scriptsize$\pm\,1.60\text{E-}2$} & $2.65\text{E-}4$\,{\scriptsize$\pm\,3.37\text{E-}5$} & $5.04\text{E-}2$\,{\scriptsize$\pm\,6.24\text{E-}3$} & $2.68\text{E-}2$\,{\scriptsize$\pm\,3.14\text{E-}3$} & $2.48\text{E-}1$\,{\scriptsize$\pm\,1.47\text{E-}2$} & $37.78$ \\
& UNO      & $9.87\text{E-}3$\,{\scriptsize$\pm\,8.52\text{E-}4$} & $5.22\text{E-}2$\,{\scriptsize$\pm\,3.30\text{E-}3$} & $3.41\text{E-}5$\,{\scriptsize$\pm\,2.70\text{E-}6$} & $1.68\text{E-}2$\,{\scriptsize$\pm\,8.40\text{E-}4$} & $8.94\text{E-}3$\,{\scriptsize$\pm\,6.90\text{E-}4$} & $8.75\text{E-}2$\,{\scriptsize$\pm\,1.13\text{E-}2$} & $37.43$ \\
& WNO      & $5.94\text{E-}2$\,{\scriptsize$\pm\,4.01\text{E-}3$} & $3.68\text{E-}1$\,{\scriptsize$\pm\,4.11\text{E-}2$} & $1.23\text{E-}3$\,{\scriptsize$\pm\,1.03\text{E-}4$} & $1.05\text{E-}1$\,{\scriptsize$\pm\,8.50\text{E-}3$} & $5.68\text{E-}2$\,{\scriptsize$\pm\,7.18\text{E-}3$} & $5.92\text{E-}1$\,{\scriptsize$\pm\,7.67\text{E-}2$} & $37.79$ \\
& PINO     & $2.43\text{E-}2$\,{\scriptsize$\pm\,2.24\text{E-}3$} & $1.18\text{E-}1$\,{\scriptsize$\pm\,1.25\text{E-}2$} & $2.07\text{E-}4$\,{\scriptsize$\pm\,1.20\text{E-}5$} & $4.49\text{E-}2$\,{\scriptsize$\pm\,3.15\text{E-}3$} & $2.38\text{E-}2$\,{\scriptsize$\pm\,3.03\text{E-}3$} & $2.19\text{E-}1$\,{\scriptsize$\pm\,2.06\text{E-}2$} & $37.78$ \\
& LNO      & $9.74\text{E-}2$\,{\scriptsize$\pm\,8.87\text{E-}3$} & $6.84\text{E-}1$\,{\scriptsize$\pm\,7.43\text{E-}2$} & $3.32\text{E-}3$\,{\scriptsize$\pm\,1.66\text{E-}4$} & $1.62\text{E-}1$\,{\scriptsize$\pm\,1.53\text{E-}2$} & $8.64\text{E-}2$\,{\scriptsize$\pm\,7.58\text{E-}3$} & $9.85\text{E-}1$\,{\scriptsize$\pm\,1.16\text{E-}1$} & $37.31$ \\
& LaMO     & $4.74\text{E-}2$\,{\scriptsize$\pm\,2.77\text{E-}3$} & $4.87\text{E-}1$\,{\scriptsize$\pm\,6.17\text{E-}2$} & $7.86\text{E-}4$\,{\scriptsize$\pm\,4.07\text{E-}5$} & $8.25\text{E-}2$\,{\scriptsize$\pm\,5.02\text{E-}3$} & $4.29\text{E-}2$\,{\scriptsize$\pm\,4.10\text{E-}3$} & $4.78\text{E-}1$\,{\scriptsize$\pm\,4.89\text{E-}2$} & $38.69$ \\
& GeoIncNO & $\mathbf{5.43\text{E-}3}$\,{\scriptsize$\pm\,1.99\text{E-}4$} & $\mathbf{2.87\text{E-}2}$\,{\scriptsize$\pm\,9.31\text{E-}4$} & $\mathbf{1.03\text{E-}5}$\,{\scriptsize$\pm\,6.29\text{E-}7$} & $\mathbf{9.24\text{E-}3}$\,{\scriptsize$\pm\,3.43\text{E-}4$} & $\mathbf{4.92\text{E-}3}$\,{\scriptsize$\pm\,2.46\text{E-}4$} & $\mathbf{4.81\text{E-}2}$\,{\scriptsize$\pm\,2.45\text{E-}3$} & $37.90$ \\

\cmidrule(lr){1-9}

\multirow{8}{*}{\textbf{3D Maxwell}}
& DeepONet & $8.74\text{E-}2$\,{\scriptsize$\pm\,6.99\text{E-}3$} & $1.15\text{E-}1$\,{\scriptsize$\pm\,9.20\text{E-}3$} & $7.68\text{E-}3$\,{\scriptsize$\pm\,6.14\text{E-}4$} & $1.25\text{E-}1$\,{\scriptsize$\pm\,1.00\text{E-}2$} & $1.29\text{E-}1$\,{\scriptsize$\pm\,1.03\text{E-}2$} & $1.58\text{E+}0$\,{\scriptsize$\pm\,1.26\text{E-}1$} & $37.80$ \\
& FNO      & $1.89\text{E-}2$\,{\scriptsize$\pm\,1.13\text{E-}3$} & $2.67\text{E-}2$\,{\scriptsize$\pm\,1.60\text{E-}3$} & $4.63\text{E-}4$\,{\scriptsize$\pm\,2.78\text{E-}5$} & $3.21\text{E-}2$\,{\scriptsize$\pm\,1.93\text{E-}3$} & $3.28\text{E-}2$\,{\scriptsize$\pm\,1.97\text{E-}3$} & $3.52\text{E-}1$\,{\scriptsize$\pm\,2.11\text{E-}2$} & $37.78$ \\
& UNO      & $2.76\text{E-}2$\,{\scriptsize$\pm\,1.93\text{E-}3$} & $3.58\text{E-}2$\,{\scriptsize$\pm\,2.51\text{E-}3$} & $9.38\text{E-}4$\,{\scriptsize$\pm\,6.57\text{E-}5$} & $4.52\text{E-}2$\,{\scriptsize$\pm\,3.16\text{E-}3$} & $4.67\text{E-}2$\,{\scriptsize$\pm\,3.27\text{E-}3$} & $5.16\text{E-}1$\,{\scriptsize$\pm\,3.61\text{E-}2$} & $37.43$ \\
& WNO      & $3.82\text{E-}2$\,{\scriptsize$\pm\,2.67\text{E-}3$} & $5.76\text{E-}2$\,{\scriptsize$\pm\,4.03\text{E-}3$} & $1.99\text{E-}3$\,{\scriptsize$\pm\,1.39\text{E-}4$} & $6.52\text{E-}2$\,{\scriptsize$\pm\,4.56\text{E-}3$} & $6.94\text{E-}2$\,{\scriptsize$\pm\,4.86\text{E-}3$} & $7.84\text{E-}1$\,{\scriptsize$\pm\,5.49\text{E-}2$} & $37.80$ \\
& PINO     & $1.68\text{E-}2$\,{\scriptsize$\pm\,1.01\text{E-}3$} & $2.32\text{E-}2$\,{\scriptsize$\pm\,1.39\text{E-}3$} & $3.59\text{E-}4$\,{\scriptsize$\pm\,2.15\text{E-}5$} & $2.79\text{E-}2$\,{\scriptsize$\pm\,1.67\text{E-}3$} & $2.87\text{E-}2$\,{\scriptsize$\pm\,1.72\text{E-}3$} & $3.05\text{E-}1$\,{\scriptsize$\pm\,1.83\text{E-}2$} & $37.78$ \\
& LNO      & $6.27\text{E-}2$\,{\scriptsize$\pm\,5.02\text{E-}3$} & $8.95\text{E-}2$\,{\scriptsize$\pm\,7.16\text{E-}3$} & $4.22\text{E-}3$\,{\scriptsize$\pm\,3.38\text{E-}4$} & $8.96\text{E-}2$\,{\scriptsize$\pm\,7.17\text{E-}3$} & $9.25\text{E-}2$\,{\scriptsize$\pm\,7.40\text{E-}3$} & $1.14\text{E+}0$\,{\scriptsize$\pm\,9.12\text{E-}2$} & $37.31$ \\
& LaMO     & $2.31\text{E-}2$\,{\scriptsize$\pm\,1.62\text{E-}3$} & $4.68\text{E-}2$\,{\scriptsize$\pm\,3.28\text{E-}3$} & $7.25\text{E-}4$\,{\scriptsize$\pm\,5.08\text{E-}5$} & $3.94\text{E-}2$\,{\scriptsize$\pm\,2.76\text{E-}3$} & $4.09\text{E-}2$\,{\scriptsize$\pm\,2.86\text{E-}3$} & $4.43\text{E-}1$\,{\scriptsize$\pm\,3.10\text{E-}2$} & $36.92$ \\
& GeoIncNO & $\mathbf{9.92\text{E-}3}$\,{\scriptsize$\pm\,4.96\text{E-}4$} & $\mathbf{1.33\text{E-}2}$\,{\scriptsize$\pm\,6.65\text{E-}4$} & $\mathbf{1.10\text{E-}4}$\,{\scriptsize$\pm\,5.50\text{E-}6$} & $\mathbf{1.47\text{E-}2}$\,{\scriptsize$\pm\,7.35\text{E-}4$} & $\mathbf{1.49\text{E-}2}$\,{\scriptsize$\pm\,7.45\text{E-}4$} & $\mathbf{1.71\text{E-}1}$\,{\scriptsize$\pm\,8.55\text{E-}3$} & $37.90$ \\

\bottomrule
\end{tabular}
}
\end{threeparttable}
\end{table*}

We evaluate GeoIncNO on six PDE benchmarks generated following the PDEBench-style data generation protocol: Burgers, KS, NS, SW equation, 3D compressible Euler (CE), and Maxwell. 
These benchmarks cover 1D, 2D, and 3D dynamical systems, including convection--diffusion dynamics, spatiotemporal chaos, incompressible vortical flow, free-surface wave propagation, compressible flow, and electromagnetic wave propagation. 

\begin{itemize}
    \item \emph{Burgers} evaluates nonlinear transport and viscous dissipation, and is governed by 
    \(\partial_t u + u\partial_x u = \nu \partial_{xx}u,\quad x\in[0,2\pi)\), 
    where \(u\) is the scalar velocity field and \(\nu\) is the viscosity.

    \item \emph{KS} evaluates long-term prediction under spatiotemporal chaos, and is governed by 
    \(\partial_t u + u\partial_x u + \partial_{xx}u + \partial_{xxxx}u = 0,\quad x\in[0,64]\), 
    where \(u\) denotes the scalar state variable.

    \item \emph{NS} evaluates incompressible vortical dynamics, and is governed by the vorticity formulation 
    \(\partial_t \omega + (\mathbf{u}\cdot\nabla)\omega = \nu\Delta\omega + f,\quad \nabla\cdot\mathbf{u}=0\), 
    where \(\omega\), \(\mathbf{u}\), \(\nu\), and \(f\) denote vorticity, velocity, viscosity, and forcing, respectively. 
    We use the unforced setting \(f=0\).

    \item \emph{SW} evaluates free-surface wave propagation and nonlinear transport, and is governed by 
    \(\partial_t h+\nabla\cdot(h\mathbf{u})=0,\quad \partial_t\mathbf{u}+(\mathbf{u}\cdot\nabla)\mathbf{u}+g\nabla h=0\), 
    where \(h\) is the fluid height, \(\mathbf{u}=(u,v)\) is the horizontal velocity field, and \(g\) is the gravitational acceleration.

    \item \emph{3D CE} evaluates density, momentum, and energy evolution under compressible dynamics, and is governed by the compressible Euler system
    \begin{equation}
    \begin{aligned}
        \partial_t \rho+\nabla\cdot(\rho\mathbf{v}) &= 0,\\
        \partial_t(\rho\mathbf{v})+\nabla\cdot(\rho\mathbf{v}\otimes\mathbf{v}+pI) &= 0,\\
        \partial_t E+\nabla\cdot((E+p)\mathbf{v}) &= 0,
    \end{aligned}
    \end{equation}
    with \(p=(\gamma-1)(E-\frac{1}{2}\rho|\mathbf{v}|^2)\), where \(\rho\), \(\mathbf{v}\), \(p\), and \(E\) denote density, velocity, pressure, and total energy, respectively.

    \item \emph{Maxwell} evaluates electromagnetic wave propagation, and is governed in a source-free vacuum by 
    \(\partial_t\mathbf{E}=c^2\nabla\times\mathbf{B},\quad \partial_t\mathbf{B}=-\nabla\times\mathbf{E}\), 
    with divergence-free constraints \(\nabla\cdot\mathbf{E}=0\) and \(\nabla\cdot\mathbf{B}=0\), where \(\mathbf{E}\), \(\mathbf{B}\), and \(c\) denote the electric field, magnetic field, and wave speed, respectively.
\end{itemize}

\subsubsection{Evaluation Metrics}
\label{subsec:evaluation-metrics}
We report all metrics averaged over the test set to evaluate complementary aspects of prediction quality.
Field-level accuracy is measured by the relative \(L_2\) error
\(\mathrm{Rel}\text{-}L_2=\|u-\hat{u}\|_2/\|u\|_2\) and MSE, where Rel-\(L_2\) provides scale-normalized accuracy and MSE reflects absolute pointwise deviations.
To assess whether the predicted solution also preserves local variation, we report the relative Sobolev error
\(\mathrm{Rel}\text{-}H^1=\|u-\hat{u}\|_{H^1}/\|u\|_{H^1}\), where the \(H^1\) norm includes both field and gradient errors.
This metric is particularly useful for derivative-sensitive dynamics, such as the Kuramoto--Sivashinsky (KS) equation.
To evaluate spectral fidelity, we use the weighted log ratio (WLR)
\(\mathrm{WLR}=\sum_{k\in\mathcal{K}} w_k \left|\log \frac{E_{\mathrm{pred}}(k)+\epsilon}{E_{\mathrm{ref}}(k)+\epsilon}\right|\).
\(E_{\mathrm{pred}}(k)=|\mathcal{F}(\hat{u})(k)|^2\) and
\(E_{\mathrm{ref}}(k)=|\mathcal{F}(u)(k)|^2\) denote the spectral energies of the predicted and reference fields at mode \(k\), respectively.
The weight \(w_k\) is computed from the normalized reference spectral energy, and \(\epsilon\) is used for numerical stability.
For time-dependent rollouts, we additionally report two RMSE-based stability metrics.
The final-step RMSE (F-RMSE) measures the prediction error at the last rollout step, whereas the cumulative RMSE (C-RMSE) measures the accumulated trajectory error over all rollout steps.
When boundary accuracy is explicitly evaluated, boundary RMSE (B-RMSE) is further used to measure the rollout error restricted to boundary regions.
To quantify long-term mean-field preservation, we further report the
relative temporal-mean drift (Mean Drift),
\(\|\mu(u)-\mu(\hat{u})\|_2/(\|\mu(u)\|_2+\epsilon)\), where
\(\mu(\cdot)=\tfrac{1}{T}\sum_{\tau=1}^{T}(\cdot)_\tau\) is the temporal
mean over the rollout; it measures the slowly accumulated bias in the
long-term average field.

In addition to RMSE, Rel-\(L_2\), and F-RMSE, we also report MAE, R$^2$, and FE to provide complementary evaluations of prediction accuracy and rollout stability.
MAE measures the average absolute deviation between the predicted and reference fields,
\(\mathrm{MAE}=\frac{1}{N}\sum_{i=1}^{N}|y_i-\hat{y}_i|\),
where it provides an intuitive measure of pointwise prediction errors.
The coefficient of determination R$^2$ evaluates the agreement between the predicted and reference solutions,
\(R^2=1-\frac{\sum_i(y_i-\hat{y}_i)^2}{\sum_i(y_i-\bar{y})^2}\),
where \(\bar{y}\) denotes the mean value of the reference solution.
A higher R$^2$ indicates that the prediction explains a larger proportion of the variance in the ground-truth field.
For autoregressive evaluation, FE measures the accumulated forecasting error during multi-step rollout and reflects the error propagation behavior over the prediction horizon.
Lower values indicate better performance for RMSE, MAE, Rel-\(L_2\), F-RMSE, and FE, while higher values indicate better performance for R$^2$.

\subsubsection{Implementation Details}
\label{app:impl}

All experiments are implemented in PyTorch and run on a single NVIDIA
RTX 4090 GPU with 24 GB memory. For each benchmark, we use 1,000 trajectories for training and 200 independently generated trajectories for testing. To keep the optimization setting of baseline neural operators consistent, all non-physics-informed baseline models are trained with the MSE prediction loss only. PINO is treated separately because it is a physics-informed baseline, and is trained with the MSE prediction loss together with the PDE residual loss following its original formulation. GeoIncNO is optimized with the objective defined in Section~\ref{subsec:training-objective}, which augments the prediction loss with active-band increment geometry regularization and projection consistency.
We use the Adam optimizer~\cite{DBLP:journals/corr/KingmaB14} with an initial learning rate of \(1\times10^{-4}\). 
A cosine annealing learning-rate scheduler is adopted to decay the learning rate during training gradually. 
Each model is trained for 100 epochs. 
To improve long-horizon stability, all models are trained using the pushforward strategy with a rollout length of \(T=5\), in which predictions are recursively fed back into the model during training. 
For GeoIncNO, we use \(M=4\) active frequency bands, set the low-rank projector dimension to \(r=8\), and set the ABP correction strength to \(\eta=0.05\) unless otherwise specified.
Across benchmarks, only the dimension-dependent operators differ:
1D/2D/3D tasks use Conv1d/2d/3d with matching FFTs, while temporal-window
processing is factored out through temporal projection, so no 4D
convolution is needed and the ABP transform acts only along spatial axes.

\subsection{Long-Horizon Rollout Stability}
\label{app:rollout}
Due to space constraints in the main text, we place the full long-horizon rollout comparison here. Figure~\ref{fig:rollout} reports the MSE, Rel-\(L_2\), and Rel-\(H^1\) of GeoIncNO and all baselines at five rollout checkpoints (\(0.2T\)--\(1.0T\)) across the six PDE benchmarks, complementing the analysis in Section~\ref{subsec:rollout}. Across every benchmark and metric, GeoIncNO attains the lowest error at each checkpoint and shows the slowest error growth over the horizon, confirming that explicitly structuring the latent transition increment suppresses long-horizon error accumulation.

\begin{figure*}[t]
  \centering
  \includegraphics[width=\textwidth]{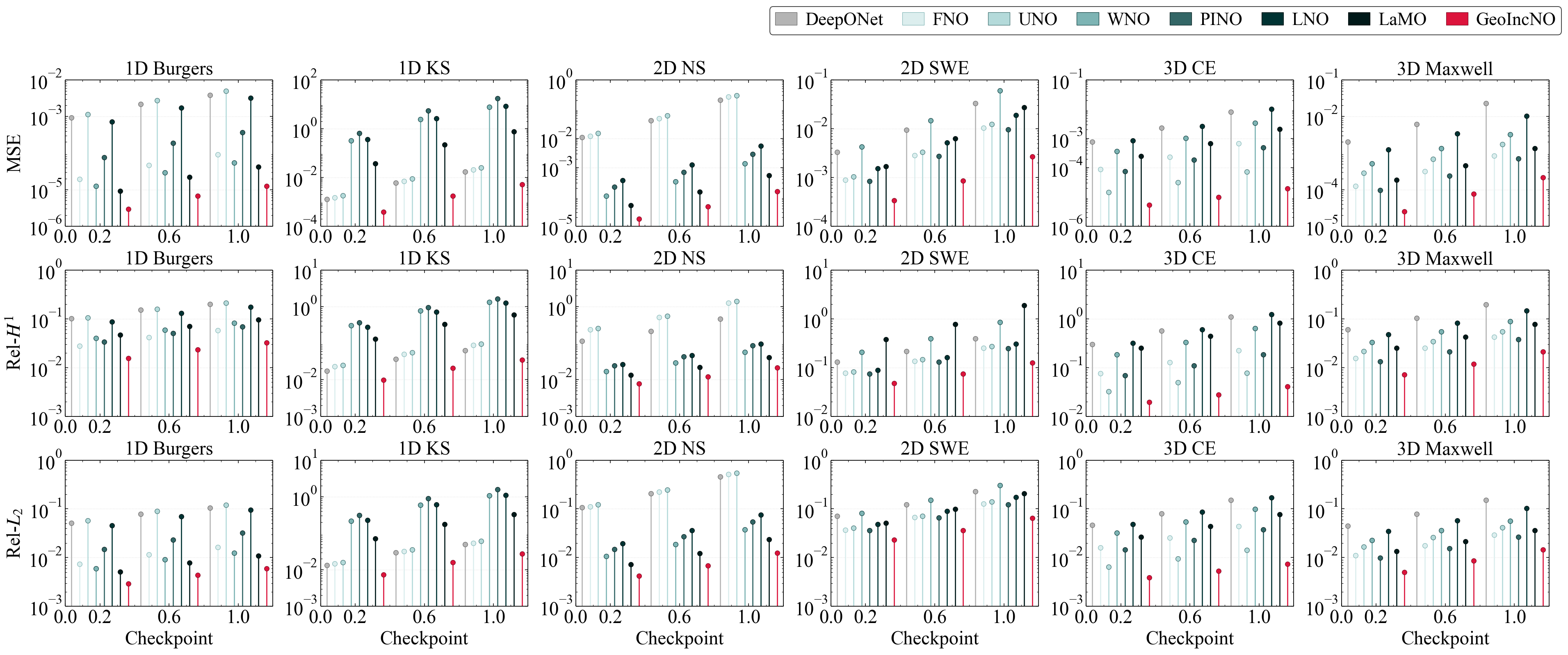}
  \caption{Long-horizon rollout errors across the six PDE benchmarks. Each column is a benchmark and each row a metric (MSE, Rel-\(H^1\), Rel-\(L_2\)); markers report the error at rollout checkpoints \(0.2T\)--\(1.0T\). GeoIncNO (red) attains the lowest error at every checkpoint and accumulates error most slowly over the horizon.}
  \label{fig:rollout}
\end{figure*}

\subsection{Cross-Parameter Generalization}
\label{subsec:cross_parameter}
To evaluate cross-parameter generalization, we conduct experiments on the 2D NS benchmark at resolution \(256\times256\).
The training set is constructed by mixing trajectories generated at two Reynolds numbers, \(\mathrm{Re}=1000\) and \(\mathrm{Re}=10000\), while testing is performed on the unseen intermediate setting \(\mathrm{Re}=5000\).
This setting evaluates whether the model can generalize to an unseen physical-parameter regime rather than only fit a fixed Reynolds-number distribution.
For each training Reynolds number, we use 500 trajectories, and the test set contains 100 independently generated trajectories at \(\mathrm{Re}=5000\).
As shown in Table~\ref{tab:cross_parameter_ns2d}, GeoIncNO achieves the strongest overall generalization to the unseen \(\mathrm{Re}=5000\) setting, obtaining the best result on every reported metric. Compared with UNO, the strongest competing baseline on the main rollout metrics, GeoIncNO reduces F-RMSE from \(7.17\mathrm{E}{-03}\) to \(3.94\mathrm{E}{-03}\), C-RMSE from \(5.99\mathrm{E}{-02}\) to \(3.30\mathrm{E}{-02}\), and WLR from \(8.46\mathrm{E}{-02}\) to \(4.65\mathrm{E}{-02}\). It also obtains the lowest Rel-\(L_2\), Rel-\(H^1\), MSE, and B-RMSE, indicating better field-level, derivative-level, and boundary-region accuracy under Reynolds-number interpolation. These results suggest that GeoIncNO improves cross-parameter robustness under Reynolds-number interpolation.
\begin{table*}[h]
\centering
\caption{Cross-parameter generalization on the 2D NS benchmark under Reynolds-number interpolation (trained on \(\mathrm{Re}=1000/10000\), tested on the unseen \(\mathrm{Re}=5000\)). Lower is better for all metrics.}
\label{tab:cross_parameter_ns2d}
\small
\setlength{\tabcolsep}{5.0pt}
\renewcommand{\arraystretch}{1.1}
\begin{tabular}{lccccccc}
\toprule
Method & F-RMSE$\downarrow$ & C-RMSE$\downarrow$ & B-RMSE$\downarrow$ & WLR$\downarrow$ & MSE$\downarrow$ & Rel-\(L_2\)$\downarrow$ & Rel-\(H^1\)$\downarrow$ \\
\midrule
\midrule
DeepONet & 3.92E-02 & 5.35E-01 & 2.90E-02 & 6.24E-01 & 1.13E-03 & 4.58E-01 & 1.06E+00 \\
FNO      & 2.12E-02 & 1.66E-01 & 6.30E-04 & 2.49E-01 & 1.52E-04 & 1.45E-01 & 3.73E-01 \\
UNO      & 7.17E-03 & 5.99E-02 & 2.81E-04 & 8.46E-02 & 1.88E-05 & 5.21E-02 & 2.49E-01 \\
WNO      & 8.44E-02 & 2.80E-01 & 1.85E-04 & 6.94E-01 & 8.54E-04 & 2.07E-01 & 1.68E+00 \\
PINO     & 1.46E-02 & 1.14E-01 & 4.49E-04 & 1.64E-01 & 7.12E-05 & 9.94E-02 & 3.02E-01 \\
LNO      & 5.56E-02 & 5.31E-01 & 9.13E-03 & 7.28E-01 & 1.30E-03 & 4.53E-01 & 1.66E+00 \\
LaMO     & 2.06E-02 & 1.55E-01 & 6.55E-04 & 2.23E-01 & 1.35E-04 & 1.34E-01 & 8.92E-01 \\
GeoIncNO & \textbf{3.94E-03} & \textbf{3.30E-02} & \textbf{1.02E-04} & \textbf{4.65E-02} & \textbf{5.42E-06} & \textbf{2.87E-02} & \textbf{1.37E-01} \\
\bottomrule
\end{tabular}
\end{table*}

\subsection{Long-Horizon Rollout Stability under Cross-Parameter Interpolation}
\label{sec:long}
To further evaluate rollout stability under an unseen physical-parameter regime, we report prediction errors at different rollout checkpoints on the cross-parameter 2D NS setting. Following the cross-parameter protocol in Appendix~\ref{subsec:cross_parameter}, all models are trained on a mixed dataset at resolution \(256\times256\) constructed from two Reynolds numbers, \(\mathrm{Re}=1000\) and \(\mathrm{Re}=10000\) (500 trajectories each), and tested on the unseen intermediate setting \(\mathrm{Re}=5000\) with 100 independently generated trajectories. The checkpoints correspond to \(0.2T\), \(0.4T\), \(0.6T\), \(0.8T\), and \(1.0T\) (rollout steps \(4\), \(7\), \(10\), \(13\), and \(16\)), which measure how errors accumulate as the autoregressive rollout proceeds under Reynolds-number interpolation.
As shown in Table~\ref{tab:long_horizon}, GeoIncNO consistently achieves the lowest errors across all rollout checkpoints. At the final checkpoint \(1.0T\), GeoIncNO reduces MSE from UNO's \(5.14\mathrm{E}{-05}\) to \(1.56\mathrm{E}{-05}\), Rel-\(L_2\) from \(1.01\mathrm{E}{-01}\) to \(5.54\mathrm{E}{-02}\), and Rel-\(H^1\) from \(4.67\mathrm{E}{-01}\) to \(2.57\mathrm{E}{-01}\). The improvement becomes more pronounced as the rollout horizon increases, indicating that GeoIncNO better suppresses error accumulation and preserves both field-level and derivative-level consistency under unseen Reynolds-number dynamics.
\begin{table*}[h]
\centering
\caption{Long-horizon rollout errors at different prediction checkpoints under Reynolds-number interpolation (trained on \(\mathrm{Re}=1000/10000\), tested on the unseen \(\mathrm{Re}=5000\)).}
\label{tab:long_horizon}
\scriptsize
\setlength{\tabcolsep}{3.0pt}
\renewcommand{\arraystretch}{1.05}
\resizebox{\textwidth}{!}{
\begin{tabular}{lccccc|ccccc|ccccc}
\toprule
\multirow{2}{*}{Model} 
& \multicolumn{5}{c|}{MSE$\downarrow$} 
& \multicolumn{5}{c|}{Rel-\(L_2\)$\downarrow$} 
& \multicolumn{5}{c}{Rel-\(H^1\)$\downarrow$} \\
\cmidrule(lr){2-6} \cmidrule(lr){7-11} \cmidrule(lr){12-16}
& \(0.2T\) & \(0.4T\) & \(0.6T\) & \(0.8T\) & \(1.0T\)
& \(0.2T\) & \(0.4T\) & \(0.6T\) & \(0.8T\) & \(1.0T\)
& \(0.2T\) & \(0.4T\) & \(0.6T\) & \(0.8T\) & \(1.0T\) \\
\midrule
\midrule
DeepONet 
& 9.09E-04 & 1.02E-03 & 1.17E-03 & 1.34E-03 & 1.54E-03
& 4.09E-01 & 4.36E-01 & 4.69E-01 & 5.06E-01 & 5.47E-01
& 1.08E+00 & 1.05E+00 & 1.04E+00 & 1.04E+00 & 1.06E+00 \\
FNO 
& 1.28E-05 & 6.16E-05 & 1.57E-04 & 2.91E-04 & 4.49E-04
& 4.91E-02 & 1.09E-01 & 1.75E-01 & 2.39E-01 & 2.99E-01
& 1.44E-01 & 2.90E-01 & 4.49E-01 & 5.98E-01 & 7.34E-01 \\
UNO 
& 2.33E-06 & 9.23E-06 & 2.06E-05 & 3.50E-05 & 5.14E-05
& 2.09E-02 & 4.19E-02 & 6.30E-02 & 8.27E-02 & 1.01E-01
& 1.14E-01 & 2.04E-01 & 2.96E-01 & 3.84E-01 & 4.67E-01 \\
WNO 
& 2.82E-06 & 2.34E-05 & 1.05E-04 & 7.14E-04 & 7.12E-03
& 2.23E-02 & 6.46E-02 & 1.41E-01 & 3.43E-01 & 8.78E-01
& 1.11E-01 & 3.09E-01 & 8.74E-01 & 2.78E+00 & 8.17E+00 \\
PINO 
& 6.26E-06 & 2.91E-05 & 7.29E-05 & 1.35E-04 & 2.13E-04
& 3.42E-02 & 7.45E-02 & 1.19E-01 & 1.63E-01 & 2.06E-01
& 1.24E-01 & 2.38E-01 & 3.61E-01 & 4.77E-01 & 5.85E-01 \\
LNO 
& 3.76E-04 & 8.43E-04 & 1.46E-03 & 2.20E-03 & 3.09E-03
& 2.65E-01 & 3.95E-01 & 5.21E-01 & 6.42E-01 & 7.65E-01
& 1.05E+00 & 1.50E+00 & 1.90E+00 & 2.26E+00 & 2.60E+00 \\
LaMO 
& 9.69E-06 & 4.87E-05 & 1.33E-04 & 2.61E-04 & 4.23E-04
& 4.23E-02 & 9.54E-02 & 1.59E-01 & 2.24E-01 & 2.88E-01
& 3.96E-01 & 7.43E-01 & 1.08E+00 & 1.37E+00 & 1.64E+00 \\
GeoIncNO 
& \textbf{7.14E-07} & \textbf{2.81E-06} & \textbf{6.24E-06} & \textbf{1.06E-05} & \textbf{1.56E-05}
& \textbf{1.15E-02} & \textbf{2.30E-02} & \textbf{3.46E-02} & \textbf{4.55E-02} & \textbf{5.54E-02}
& \textbf{6.25E-02} & \textbf{1.12E-01} & \textbf{1.63E-01} & \textbf{2.11E-01} & \textbf{2.57E-01} \\
\bottomrule
\end{tabular}
}
\end{table*}

\begin{table*}[t]
\centering
\caption{Cross-resolution generalization results on the 2D NS benchmark. Models are trained at \(32\times32\) and tested at the training resolution and finer resolutions. Lower values indicate better performance.}
\label{tab:cross_resolution_ns2d}
\small
\setlength{\tabcolsep}{3.2pt}
\renewcommand{\arraystretch}{1.05}
\begin{tabular}{llccccccc}
\toprule
Resolution & Method 
& MSE@1$\downarrow$ 
& Rel-\(L_2\)@1$\downarrow$ 
& Rel-\(H^1\)@1$\downarrow$ 
& MSE@16$\downarrow$ 
& nRMSE@16$\downarrow$ 
& Rel-\(L_2\)@16$\downarrow$ 
& Rel-\(H^1\)@16$\downarrow$ \\
\midrule
\midrule
\multirow{8}{*}{\(32\times32\)}
& DeepONet & 9.97E-04 & 3.44E-01 & 6.91E-01 & 1.14E-03 & 9.21E-02 & 4.19E-01 & 7.55E-01 \\
& FNO      & 2.95E-07 & 5.91E-03 & 2.00E-02 & 1.26E-05 & 9.69E-03 & 4.41E-02 & 1.31E-01 \\
& UNO      & 3.77E-07 & 6.68E-03 & 2.23E-02 & 1.59E-05 & 1.09E-02 & 4.95E-02 & 1.34E-01 \\
& WNO      & 1.99E-07 & 4.86E-03 & 1.66E-02 & 8.09E-05 & 2.45E-02 & 1.12E-01 & 2.94E-01 \\
& PINO     & 3.48E-07 & 6.41E-03 & 2.15E-02 & 1.04E-05 & 8.81E-03 & 4.01E-02 & 1.16E-01 \\
& LNO      & 7.27E-05 & 9.28E-02 & 2.87E-01 & 1.90E-03 & 1.19E-01 & 5.41E-01 & 9.90E-01 \\
& LaMO     & 5.75E-07 & 8.25E-03 & 2.95E-02 & 2.01E-04 & 3.87E-02 & 1.76E-01 & 4.89E-01 \\
& GeoIncNO & \textbf{6.03E-08} & \textbf{2.67E-03} & \textbf{9.12E-03} & \textbf{3.15E-06} & \textbf{4.85E-03} & \textbf{2.20E-02} & \textbf{6.38E-02} \\
\cmidrule(lr){1-9}
\multirow{8}{*}{\(64\times64\)}
& DeepONet & 1.02E-03 & 3.46E-01 & 7.59E-01 & 1.16E-03 & 9.28E-02 & 4.22E-01 & 8.08E-01 \\
& FNO      & 7.48E-07 & 9.38E-03 & 7.16E-02 & 1.48E-05 & 1.05E-02 & 4.76E-02 & 2.18E-01 \\
& UNO      & 8.90E-07 & 1.02E-02 & 7.62E-02 & 2.08E-05 & 1.24E-02 & 5.66E-02 & 2.29E-01 \\
& WNO      & 3.40E-06 & 2.00E-02 & 1.13E-01 & 4.77E-04 & 5.96E-02 & 2.71E-01 & 9.14E-01 \\
& PINO     & 8.25E-07 & 9.85E-03 & 7.61E-02 & 1.23E-05 & 9.58E-03 & 4.36E-02 & 1.96E-01 \\
& LNO      & 8.16E-05 & 9.80E-02 & 3.80E-01 & 2.00E-03 & 1.22E-01 & 5.54E-01 & 1.15E+00 \\
& LaMO     & 7.45E-06 & 2.96E-02 & 1.53E-01 & 1.74E-04 & 3.60E-02 & 1.64E-01 & 5.62E-01 \\
& GeoIncNO & \textbf{2.26E-07} & \textbf{5.16E-03} & \textbf{3.94E-02} & \textbf{3.73E-06} & \textbf{5.27E-03} & \textbf{2.40E-02} & \textbf{1.08E-01} \\
\cmidrule(lr){1-9}
\multirow{8}{*}{\(128\times128\)}
& DeepONet & 1.03E-03 & 3.47E-01 & 8.42E-01 & 1.17E-03 & 9.30E-02 & 4.23E-01 & 8.61E-01 \\
& FNO      & 1.03E-06 & 1.10E-02 & 1.37E-01 & 1.58E-05 & 1.08E-02 & 4.91E-02 & 2.73E-01 \\
& UNO      & 1.19E-06 & 1.18E-02 & 1.37E-01 & 2.43E-05 & 1.34E-02 & 6.09E-02 & 3.00E-01 \\
& WNO      & 1.11E-05 & 3.60E-02 & 2.65E-01 & 8.74E-04 & 8.04E-02 & 3.65E-01 & 1.53E+00 \\
& PINO     & 1.09E-06 & 1.13E-02 & 1.36E-01 & 1.32E-05 & 9.87E-03 & 4.49E-02 & 2.50E-01 \\
& LNO      & 8.54E-05 & 1.00E-01 & 4.45E-01 & 2.04E-03 & 1.23E-01 & 5.58E-01 & 1.39E+00 \\
& LaMO     & 8.49E-06 & 3.15E-02 & 2.49E-01 & 1.80E-04 & 3.65E-02 & 1.66E-01 & 6.78E-01 \\
& GeoIncNO & \textbf{3.13E-07} & \textbf{6.05E-03} & \textbf{7.46E-02} & \textbf{4.19E-06} & \textbf{5.57E-03} & \textbf{2.53E-02} & \textbf{1.37E-01} \\
\cmidrule(lr){1-9}
\multirow{8}{*}{\(256\times256\)}
& DeepONet & 1.04E-03 & 3.47E-01 & 9.69E-01 & 1.19E-03 & 9.36E-02 & 4.26E-01 & 9.21E-01 \\
& FNO      & 1.11E-06 & 1.14E-02 & 1.74E-01 & 1.45E-05 & 1.03E-02 & 4.69E-02 & 3.05E-01 \\
& UNO      & 1.28E-06 & 1.22E-02 & 1.72E-01 & 2.52E-05 & 1.36E-02 & 6.18E-02 & 3.65E-01 \\
& WNO      & 1.46E-05 & 4.12E-02 & 3.11E-01 & 9.63E-04 & 8.40E-02 & 3.82E-01 & 1.72E+00 \\
& PINO     & 1.17E-06 & 1.16E-02 & 1.65E-01 & 1.32E-05 & 9.86E-03 & 4.48E-02 & 3.01E-01 \\
& LNO      & 9.01E-05 & 1.02E-01 & 5.15E-01 & 2.16E-03 & 1.26E-01 & 5.72E-01 & 1.62E+00 \\
& LaMO     & 8.82E-06 & 3.20E-02 & 3.60E-01 & 1.91E-04 & 3.74E-02 & 1.70E-01 & 8.20E-01 \\
& GeoIncNO & \textbf{3.69E-07} & \textbf{6.55E-03} & \textbf{9.45E-02} & \textbf{4.53E-06} & \textbf{5.76E-03} & \textbf{2.62E-02} & \textbf{1.67E-01} \\
\bottomrule
\end{tabular}
\end{table*}

\subsection{Cross-Resolution Generalization}
\label{subsec:cross-resolution}
We further evaluate cross-resolution generalization on the 2D NS benchmark.
All models are trained at a low spatial resolution of \(32\times32\) and tested at the training resolution \(32\times32\) as well as finer resolutions \(64\times64\), \(128\times128\), and \(256\times256\).
This setting examines whether a model trained on coarse-grid data can transfer to finer spatial discretizations while maintaining stable autoregressive rollout.
We evaluate both one-step prediction and long-horizon rollout performance.
For one-step prediction, we report MSE@1, Rel-\(L_2\)@1, and Rel-\(H^1\)@1, which measure the immediate prediction error after one rollout step from absolute, relative, and gradient-aware perspectives.
For long-horizon prediction, we report MSE@16, nRMSE@16, Rel-\(L_2\)@16, and Rel-\(H^1\)@16, which measure the prediction quality after 16 autoregressive rollout steps.
Thus, the @1 metrics reflect short-term resolution transfer, while the @16 metrics reflect long-horizon stability under cross-resolution evaluation.
As shown in Table~\ref{tab:cross_resolution_ns2d}, GeoIncNO achieves the best performance at the training resolution \(32\times32\) across both one-step and 16-step rollout metrics, with MSE@1 of \(6.03\mathrm{E}{-08}\), Rel-\(L_2\)@1 of \(2.67\mathrm{E}{-03}\), MSE@16 of \(3.15\mathrm{E}{-06}\), and Rel-\(L_2\)@16 of \(2.20\mathrm{E}{-02}\).
When transferring to finer resolutions, GeoIncNO continues to attain the lowest error on every reported metric, covering both the one-step (@1) and 16-step (@16) settings, indicating that its advantage is preserved under cross-resolution transfer rather than being limited to the training resolution.
At the highest tested resolution \(256\times256\), GeoIncNO obtains MSE@16 of \(4.53\mathrm{E}{-06}\), nRMSE@16 of \(5.76\mathrm{E}{-03}\), and Rel-\(L_2\)@16 of \(2.62\mathrm{E}{-02}\), and it also achieves the lowest Rel-\(H^1\)@16 of \(1.67\mathrm{E}{-01}\), maintaining both field-level and derivative-level accuracy after repeated rollout.
These results suggest that the proposed increment-centered modeling is beneficial for suppressing long-horizon error accumulation under cross-resolution transfer.

\begin{table*}[t]
\centering
\caption{Performance comparison on the FSI task from RealPDEBench under different training paradigms. Results of baseline models are taken from the original RealPDEBench paper~\cite{hu2026realpdebench}.}
\label{tab:realpdebench_fsi}
\begin{tabular}{ll|ccc|ccc|ccc}
\toprule
\multirow{2}{*}{Baseline} & \multirow{2}{*}{Params} 
& \multicolumn{3}{c|}{Simulated Training}
& \multicolumn{3}{c|}{Real-world Training}
& \multicolumn{3}{c}{Real-world Finetuning} \\
\cmidrule(lr){3-5}
\cmidrule(lr){6-8}
\cmidrule(lr){9-11}
& 
& RMSE$\downarrow$ & Rel-\(L_2\)$\downarrow$ & F-RMSE$\downarrow$
& RMSE$\downarrow$ & Rel-\(L_2\)$\downarrow$ & F-RMSE$\downarrow$
& RMSE$\downarrow$ & Rel-\(L_2\)$\downarrow$ & F-RMSE$\downarrow$ \\
\midrule
\midrule
U-Net~\cite{ronneberger2015u} & 23.0M 
& 2.23E-02 & 1.59E-01 & \textbf{2.50E-03}
& \textbf{8.50E-03} & 5.83E-02 & 7.00E-04
& 8.40E-03 & 5.79E-02 & \textbf{7.00E-04} \\
CNO~\cite{raonic2023convolutional} & 8.0M
& 2.41E-02 & 1.72E-01 & 2.90E-03
& 1.05E-02 & 7.41E-02 & 9.00E-04
& 9.60E-03 & 6.79E-02 & 8.00E-04 \\
DeepONet~\cite{lu2021deeponet} & 3.4M
& 6.37E-02 & 4.61E-01 & 9.70E-03
& 3.50E-02 & 2.50E-01 & 5.10E-03
& 3.32E-02 & 2.37E-01 & 4.80E-03 \\
FNO~\cite{DBLP:conf/iclr/LiKALBSA21} & 268.5M
& 4.26E-02 & 3.10E-01 & 5.90E-03
& 1.29E-02 & 8.92E-02 & 1.20E-03
& 1.27E-02 & 8.81E-02 & 1.20E-03 \\
WDNO~\cite{hu2025wavelet} & 91.7M
& 3.69E-02 & 2.70E-01 & 5.00E-03
& 1.17E-02 & 8.17E-02 & 1.10E-03
& 1.16E-02 & 8.24E-02 & 1.10E-03 \\
MWT~\cite{gupta2021multiwavelet} & 2.9M
& 3.39E-02 & 2.49E-01 & 4.60E-03
& 1.28E-02 & 9.10E-02 & 1.00E-03
& 1.28E-02 & 9.12E-02 & 1.00E-03 \\
GK-Transformer~\cite{cao2021choose} & 67.6M
& 3.07E-02 & 2.24E-01 & 3.90E-03
& 1.27E-02 & 9.16E-02 & 1.10E-03
& 1.24E-02 & 8.98E-02 & 1.00E-03 \\
Transolver~\cite{wu2024transolver} & 4.3M
& 3.05E-02 & 2.12E-01 & 4.30E-03
& 2.36E-02 & 1.57E-01 & 2.70E-03
& 2.23E-02 & 1.48E-01 & 2.50E-03 \\
DPOT-S-FT~\cite{hao2024dpot} & 41.3M
& 2.60E-02 & 1.89E-01 & 3.10E-03
& 1.05E-02 & 7.46E-02 & 8.00E-04
& 9.90E-03 & 7.01E-02 & \textbf{7.00E-04} \\
DPOT-L-FT~\cite{hao2024dpot} & 673.5M
& 2.62E-02 & 1.90E-01 & 3.00E-03
& 9.90E-03 & 6.87E-02 & 8.00E-04
& 9.70E-03 & 6.68E-02 & 8.00E-04 \\
ML Average & --
& 3.37E-02 & 2.43E-01 & 4.50E-03
& 1.48E-02 & 1.04E-01 & 1.50E-03
& 1.43E-02 & 9.99E-02 & 1.50E-03 \\
DMD~\cite{kutz2016dynamic} & --
& -- & -- & --
& -- & -- & --
& 4.50E-02 & 2.96E-01 & 6.00E-03 \\
GeoIncNO & 9.8M
& \textbf{1.91E-02} & \textbf{1.46E-01} & 2.60E-03
& \textbf{8.50E-03} & \textbf{5.31E-02} & \textbf{6.40E-04}
& \textbf{8.18E-03} & \textbf{5.05E-02} & \textbf{7.00E-04} \\
\bottomrule
\end{tabular}
\end{table*}

\begin{table*}[t]
\centering
\caption{Autoregressive evaluation on the FSI task from RealPDEBench. Results of baseline models are taken from the original RealPDEBench paper~\cite{hu2026realpdebench}.}
\label{tab:realpdebench_fsi_ar}
\renewcommand{\arraystretch}{1.05}
\begin{tabular}{llcccccc}
\toprule
Model & Params & RMSE$\downarrow$ & MAE$\downarrow$ & Rel-\(L_2\)$\downarrow$ & R$^2$$\uparrow$ & F-RMSE$\downarrow$ & FE$\downarrow$ \\
\midrule
\midrule
U-Net~\cite{ronneberger2015u} & 23.0M 
& 1.76E-02 & \textbf{8.60E-03} & 1.19E-01 & 9.47E-01 & 1.23E-03 & 2.00E+01 \\
CNO~\cite{raonic2023convolutional} & 8.0M
& 2.27E-02 & 1.19E-02 & 1.55E-01 & 9.15E-01 & 1.51E-03 & 5.19E+01 \\
DeepONet~\cite{lu2021deeponet} & 3.4M
& 3.63E-02 & 2.21E-02 & 2.59E-01 & 7.74E-01 & 2.61E-03 & 7.68E+01 \\
FNO~\cite{DBLP:conf/iclr/LiKALBSA21} & 268.5M
& 2.07E-02 & 1.11E-02 & 1.41E-01 & 9.26E-01 & 1.49E-03 & 3.54E+01 \\
WDNO~\cite{hu2025wavelet} & 91.7M
& 2.15E-02 & 1.16E-02 & 1.49E-01 & 9.21E-01 & 1.53E-03 & 4.41E+01 \\
MWT~\cite{gupta2021multiwavelet} & 2.9M
& 2.97E-02 & 1.44E-02 & 2.12E-01 & 8.56E-01 & 1.78E-03 & 4.01E+01 \\
GK-Transformer~\cite{cao2021choose} & 67.6M
& 2.06E-02 & 1.12E-02 & 1.42E-01 & 9.27E-01 & 1.45E-03 & 3.21E+01 \\
Transolver~\cite{wu2024transolver} & 4.3M
& 2.35E-02 & 1.28E-02 & 1.57E-01 & 9.02E-01 & 2.67E-03 & \textbf{1.31E+01} \\
DPOT-S-FT~\cite{hao2024dpot} & 41.3M
& 1.86E-02 & 9.54E-03 & 1.27E-01 & 9.40E-01 & 1.30E-03 & 2.10E+01 \\
DPOT-L-FT~\cite{hao2024dpot} & 673.5M
& 1.79E-02 & 8.91E-03 & 1.20E-01 & 9.45E-01 & 1.25E-03 & 1.74E+01 \\
DMD~\cite{kutz2016dynamic} & --
& 6.25E-02 & 3.29E-02 & 4.13E-01 & 3.12E-01 & 4.58E-03 & 8.22E+01 \\
GeoIncNO & 9.8M
& \textbf{1.64E-02} & 8.82E-03 & \textbf{1.10E-01} & \textbf{9.79E-01} & \textbf{1.16E-03} & 1.99E+01 \\
\bottomrule
\end{tabular}
\end{table*}

\subsection{RealPDEBench FSI Evaluation}
\label{app:realpdebench}

We conduct experiments on the fluid-structure interaction (FSI) task from RealPDEBench~\cite{hu2026realpdebench}, which involves coupled fluid and structural dynamics.
Following the RealPDEBench protocol, we compare models under three training paradigms: simulated training, real-world training, and real-world finetuning.
We report RMSE, MAE, Rel-\(L_2\), R$^2$, and F-RMSE, where lower values indicate better performance for error-based metrics and higher values indicate better performance for R$^2$.
The detailed definitions of all evaluation metrics are provided in Appendix~\ref{subsec:evaluation-metrics}.
DMD is included as the original non-training baseline and is reported only under the corresponding inference setting.

As shown in Table~\ref{tab:realpdebench_fsi}, GeoIncNO achieves strong performance with only 9.8M parameters.
Under simulated training, GeoIncNO obtains the lowest RMSE and Rel-\(L_2\).
Under real-world training, GeoIncNO achieves the lowest Rel-\(L_2\) and F-RMSE while matching the best RMSE.
Under real-world finetuning, GeoIncNO obtains the lowest RMSE and Rel-\(L_2\) and matches the best F-RMSE.
These results show that GeoIncNO provides consistent gains across different RealPDEBench training settings.

Table~\ref{tab:realpdebench_fsi_ar} further reports autoregressive evaluation results.
GeoIncNO achieves the best RMSE, Rel-\(L_2\), R$^2$, and F-RMSE among all compared models.
Although U-Net obtains a slightly lower MAE and Transolver achieves the lowest FE, GeoIncNO provides the strongest overall rollout performance across the main prediction-accuracy metrics.
Compared with DPOT-L-FT, GeoIncNO reduces RMSE from \(1.79\mathrm{E}{-02}\) to \(1.64\mathrm{E}{-02}\) and Rel-\(L_2\) from \(1.20\mathrm{E}{-01}\) to \(1.10\mathrm{E}{-01}\) while using far fewer parameters, 9.8M versus 673.5M.
These results indicate that GeoIncNO maintains stable multi-step prediction behavior on the real-world FSI task.

\subsection{Qualitative Visualization}
\label{app:qualitative}

We further provide qualitative visualizations on competitive benchmarks, including 1D Burgers, 2D NS, 2D SW, and 3D Maxwell.
These examples cover different types of PDE dynamics, such as nonlinear transport, vortical flow, free-surface wave propagation, and electromagnetic wave propagation.
The visual comparisons are used to examine structural preservation, phase consistency, and error accumulation during prediction.

\paragraph{1D Burgers}
Figure~\ref{fig:vis_overall} compares the predicted solution fields and corresponding error heatmaps on the 1D Burgers benchmark.
GeoIncNO closely follows the smooth spatiotemporal structure of the ground-truth solution, while several baselines exhibit visible distortions along the temporal direction.
The error maps further show that GeoIncNO produces weaker and more uniformly distributed errors, whereas the baselines contain larger localized error regions.
These results indicate that GeoIncNO better captures nonlinear transport dynamics in the Burgers equation.

\paragraph{2D NS and SW}
Figure~\ref{fig:ns_swe_visualization} presents qualitative rollout results on the 2D NS and SW benchmarks at \(t=0,5,15,20\).
For NS, GeoIncNO better preserves vortical structures and avoids the structural distortion or misplaced high-vorticity regions observed in several baselines.
For SW, GeoIncNO maintains more accurate fluid-height evolution and reduces the phase mismatch, amplitude distortion, and enlarged spatial errors produced by other methods.
These results suggest that GeoIncNO provides more stable rollout predictions for two-dimensional dynamical systems.

\paragraph{3D Maxwell}
Figure~\ref{fig:maxwell_comparison} visualizes the rollout predictions of the magnetic-field component \(B_x\) and the electric-field component \(E_x\) on the 3D Maxwell benchmark.
For each component, XY, XZ, and YZ slices are shown at \(t=0,5,10,15,20\), together with the ground truth, prediction, and absolute error maps.
GeoIncNO preserves the main wave structures across different slice planes and maintains small errors throughout the rollout.
By contrast, several baselines exhibit over-smoothed fields, distorted wave patterns, or enlarged errors at later time steps.
These results show that GeoIncNO better preserves 3D electromagnetic dynamics and phase consistency under long-horizon prediction.

\clearpage
\begin{figure*}[t]
    \centering
    \includegraphics[width=\textwidth]{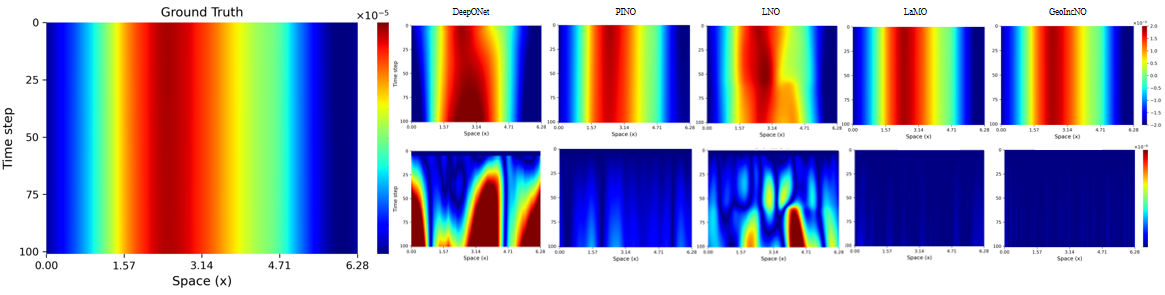}
    \caption{Qualitative comparison on the 1D Burgers benchmark, including the ground-truth solution, predicted solution fields, and absolute error maps.}
    \label{fig:vis_overall}
\end{figure*}

\begin{figure*}[b]
    \centering
    \subfigure[2D NS]{
        \includegraphics[width=0.98\linewidth]{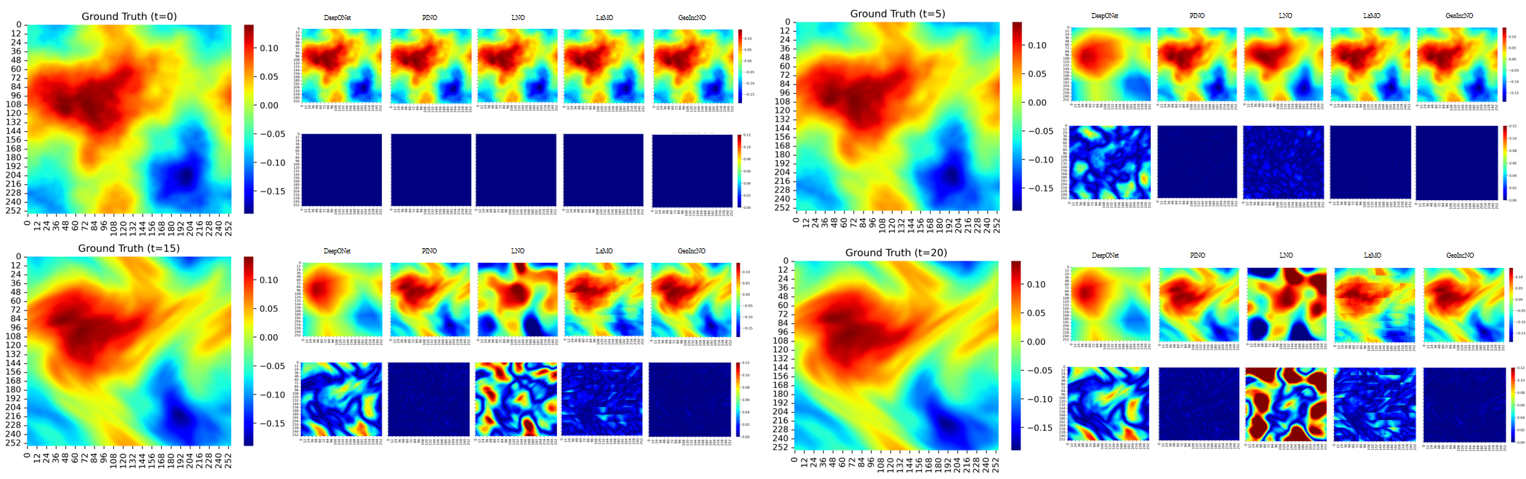}
        \label{fig:ns_visualization}
    }

    \vspace{0.3em}
    \subfigure[2D SW]{
        \includegraphics[width=0.98\linewidth]{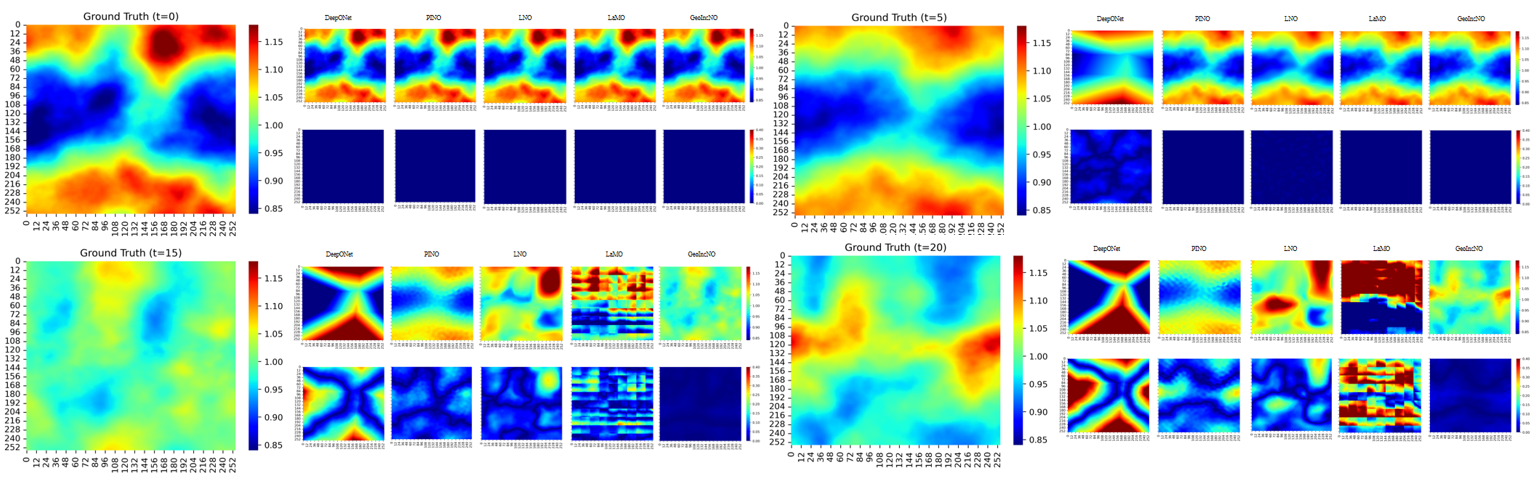}
        \label{fig:swe_visualization}
    }
    \caption{Qualitative rollout visualization on 2D dynamical benchmarks at \(t=0,5,15,20\). 
    The top panel shows vorticity-field predictions on the 2D NS benchmark, and the bottom panel shows fluid-height predictions on the 2D SW benchmark.}
    \label{fig:ns_swe_visualization}
\end{figure*}

\begin{figure*}[t]
    \centering
    \subfigure[$B_x$ component]{
        \includegraphics[width=0.46\linewidth]{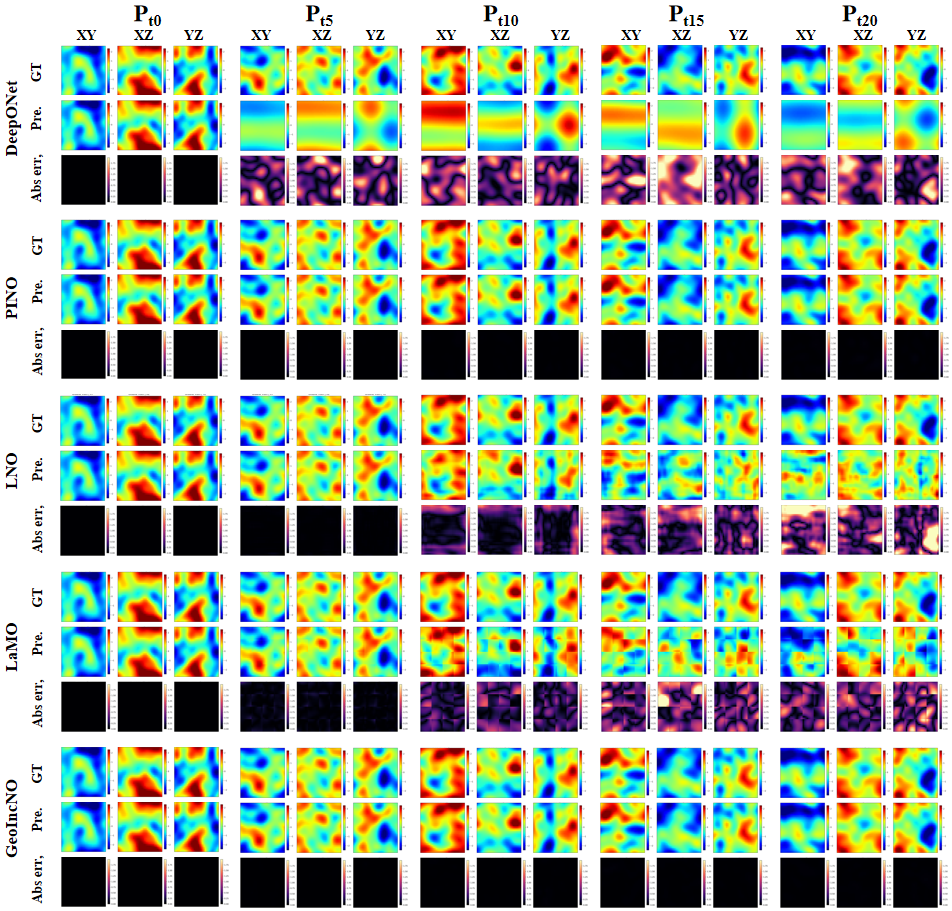}
    }
    \hfill
    \subfigure[$E_x$ component]{
        \includegraphics[width=0.47\linewidth]{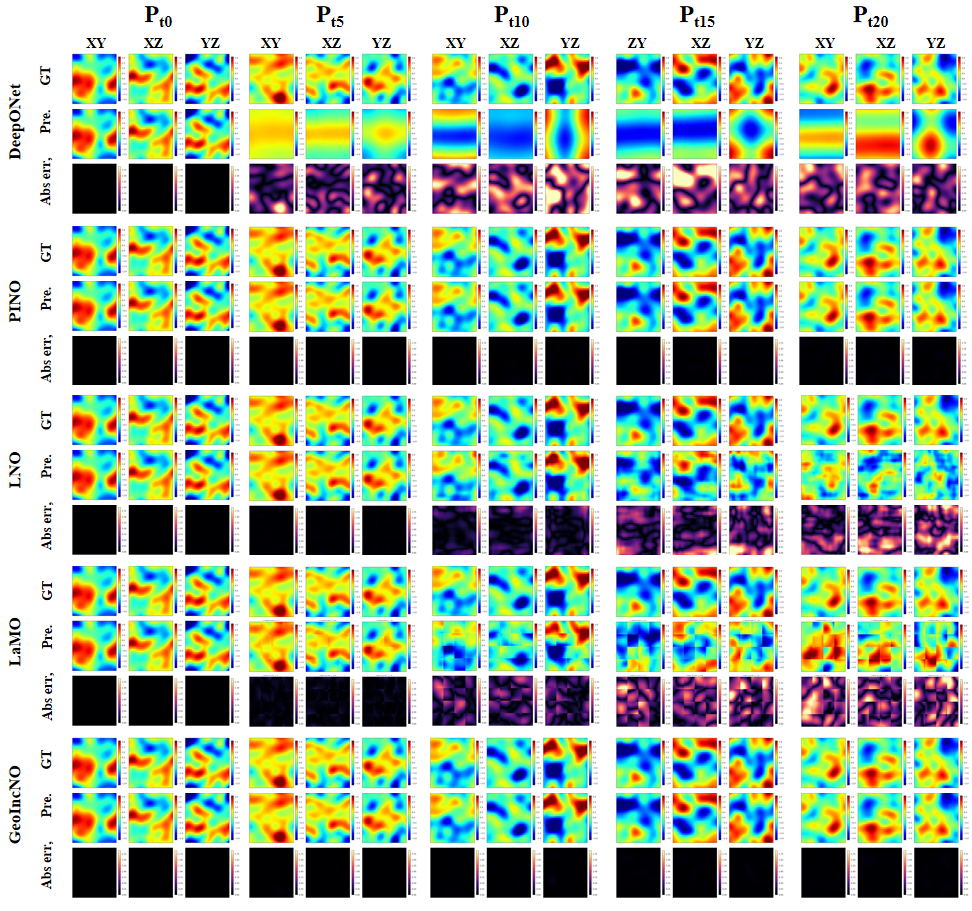}
    }
    \caption{Qualitative comparison on the 3D Maxwell benchmark. 
    The left panel shows slice-wise predictions and absolute error maps for the magnetic field component \(B_x\), and the right panel shows those for the electric field component \(E_x\).}
    \label{fig:maxwell_comparison}
\end{figure*}

\end{document}